\def\year{2019}\relax
\documentclass[letterpaper]{article}
\usepackage{aaai19}
\usepackage{times}
\usepackage{helvet}
\usepackage{courier}
\usepackage{url}
\usepackage{hyperref}
\usepackage{graphicx}
\usepackage{hyperref}
\usepackage{booktabs}
\usepackage{amsfonts}
\usepackage{nicefrac}
\usepackage{microtype}
\usepackage{amsmath}
\usepackage{multicol}
\usepackage{multirow}
\usepackage{newtxtext}
\usepackage{xr}
\usepackage[numbers]{natbib}
\usepackage{subfigure}
\usepackage{algorithmic}
\usepackage[boxruled, linesnumbered]{algorithm2e}
\usepackage{bm}
\usepackage{amsthm}
\usepackage{verbatim}

\newtheorem{thm}{Theorem}

\newcommand{\argmax}{\mathop{\mathrm{arg~max}}\limits}
\newcommand{\argmin}{\mathop{\mathrm{arg~min}}\limits}

\begin{document}
%
\title{Bayesian posterior approximation via greedy particle optimization \thanks{A longer version including Appendix is available at https://arxiv.org/abs/1805.07912}}

\author{
  Futoshi Futami\\
  The University of Tokyo, RIKEN\\
  \texttt{futami@ms.k.u-tokyo.ac.jp }\\
   \And
   Zhenghang Cui\\
   The University of Tokyo, RIKEN\\
   \texttt{cui@ms.k.u-tokyo.ac.jp}\\
   \AND
   Issei Sato \\
   The University of Tokyo, RIKEN\\
   \texttt{sato@k.u-tokyo.ac.jp}\\
   \And
  Masashi Sugiyama \\
  RIKEN, The University of Tokyo\\
   \texttt{sugi@k.u-tokyo.ac.jp} \\
}

\maketitle
\begin{abstract}
In Bayesian inference, the posterior distributions are difficult to obtain analytically for complex models such as neural networks. Variational inference usually uses a parametric distribution for approximation, from which we can easily draw samples.
Recently discrete approximation by particles has attracted attention because of its high expression ability.
An example is Stein variational gradient descent (SVGD), which iteratively optimizes particles. Although SVGD has been shown to be computationally efficient empirically, its theoretical properties
have not been clarified yet and no finite sample bound of the convergence rate is known. Another example is the Stein points (SP) method, which minimizes kernelized Stein discrepancy directly. Although a finite sample bound is assured theoretically, SP is computationally inefficient empirically, especially in high-dimensional problems. In this paper, we propose a novel method named \emph{maximum mean discrepancy minimization by the Frank-Wolfe algorithm (MMD-FW)}, which minimizes MMD in a greedy way by the FW algorithm. Our method is computationally efficient empirically and we show that its finite sample convergence bound is in a linear order in finite dimensions.
\end{abstract}

\section{Introduction}
In Bayesian inference, approximating the posterior distribution $p(x)$ over parameter $x$ is the most important task in general. When we express the prior distribution as $p_0(x)$ and the likelihood as $L(D|x)$ where $D$ denotes observations, the posterior distribution can be obtained up to a constant factor as $p(x)\propto L(D|x)p_0(x)$. 
 In many cases, analytical expression of the normalizing constant cannot be obtained; thus, we need an approximated posterior $\hat{p}(x)$, which can be used, e.g., for calculating the predictive distribution \cite{PRML}:
\begin{align}\label{integral_0}
Z_{L,\hat{p}}=\int L(y|x)\hat{p}(x)dx,
\end{align}
where $L(y|x)$ is the likelihood function of a new observation $y$ given parameter $x$.
Variational inference (VI) is widely used as an approximation method for the posterior distribution \cite{blei2017variational}. VI approximates the target distribution with a parametric distribution, from which we can easily draw samples.
In VI, we often consider the mean field assumption and use parametric models in the exponential family \cite{blei2017variational}. Although these assumptions are used to make optimization computationally tractable, they are often too restrictive to approximate the target distribution well. Therefore, the approximate distribution never converges to the target distribution in general, which means that the approximation of Eq.(\ref{integral_0}) is biased, and no theoretical guarantee is assured.

An alternative way is a discrete approximation of the target distribution by using a set of particles \cite{PRML}, $\hat{p}(x)=\sum_{n=1}^N \delta(x,x_n)/N$, where $N$ is the number of particles and $\delta$ is the Dirac delta function. Particle approximation is free of VI assumptions and thus is more expressive. The Monte Carlo (MC) method is typically used to draw particles randomly and independently \cite{PRML}. However, the drawbacks of MC are that vast computational resources are required to sample from multi-modal and high-dimensional distributions, and it is hard to estimate when to stop the algorithm.

Recently, methods that optimize particles through iterative updates have been explored. A representative example is Stein variational gradient descent (SVGD) \cite{liu2016stein}, which iteratively updates all particles in the direction that is characterized by kernelized Stein discrepancy (KSD). The update is actually implemented by gradient descent and SVGD empirically works well in high-dimensional problems. 
However, theoretical properties of SVGD have not been clarified and no finite sample bound of the convergence rate is known \cite{liu2017stein}.
Another example is the Stein points (SP) method \cite{chen2018stein}, which directly minimizes KSD. Although this method is assured by a finite sample convergence bound, it is not practically feasible in high-dimensional problems due to the curse of dimensionality, because gradient descent is not available and sampling or grid search needs to be used for optimization. Moreover, the number of evaluations of the gradient of the log probability, which usually requires vast computation costs, is four times that of SVGD.

We aim to develop a discrete approximation method that is computationally efficient, works well in high-dimensional problems, and also has a theoretical guarantee for the convergence rate.
In this paper, we propose \emph{maximum mean discrepancy minimization by the Frank-Wolfe  algorithm (MMD-FW)} in a greedy way.
Our convex formulation of discrete approximation enables us to use the Frank-Wolfe (FW) algorithm \cite{jaggi2013revisiting} and to derive a finite sample bound of the convergence rate.

Our contributions in this paper are three-fold:\\
1. We formulate a discrete approximation method in terms of convex optimization of MMD in a reproducing kernel Hilbert space (RKHS), and solve it with the FW algorithm.\\
2. Our algorithm is computationally efficient and empirically works well in high-dimensional problems. It has a guaranteed finite sample bound of the convergence rate. \\
3. We show empirically that our method compares favorably with existing particle optimization methods.

\section{Preliminary}
In this section, we review two existing particle optimization methods, SVGD and SP. After that, we introduce MMD which is our objective function. We assume that $x \in \mathbb{R}^d$ and let $k:X\times X\rightarrow\mathbb{R}$ be the reproducing kernel of an RKHS $\mathcal{H}$ of functions $X\rightarrow\mathbb{R}$ with the inner product $\langle \cdot,\cdot \rangle$ and $\|\cdot\|_\mathcal{H}$ is the assosiated norm, where $X\subseteq \mathbb{R}^d$ denotes the input domain.

\subsection{Stein variational gradient descent (SVGD)}
We first prepare initial particles $\hat{p}_0(x)=\sum_{n=1}^N \delta(x,x_n)/N$ and iteratively update them by a transformation, $T(x)=x+\epsilon \phi(x)$, where $\phi(x)$ is a perturbation direction. When the current empirical distribution is $\hat{p}(x)=\sum_{n=1}^N \delta(x,x_n)/N$, then $\phi(x)$ is chosen to maximally decrease the Kullback-Leibler (KL) divergence between the empirical distribution $\hat{p}$ formed by the particles and the target distribution $p$, $\phi^*(x)=\argmax_{\phi \in \mathcal{F}}\left\{-\frac{d}{d\epsilon}\mathrm{KL}(\hat{p}_{[T]}\|p)|_{\epsilon=0}\right\}$, where $\mathcal{F}$ denotes a set of candidate functions from which we choose map $\phi$, and $\hat{p}_{[T]}(z) = \hat{p}(T^{-1}(z))\cdot|\det(\nabla_zT^{-1}(z))|$. Liu et al. \cite{liu2016stein} proved that this problem is characterized by the Stein operator, $-\frac{d}{d\epsilon}\mathrm{KL}(\hat{p}_{[\epsilon \phi]}\|p)|_{\epsilon=0}=\mathbb{E}_{x\sim \hat{p}}[\mathcal{S}_p\phi(x)]$, where $\mathcal{S}_p$ denotes the Stein operator $\mathcal{S}_p\phi(x)=\nabla\ln p(x)\phi(x)^\top+\nabla\cdot\phi(x)$ which acts on a $d \times 1$ vector function $\phi$ and returns a scalar value function.
Thus, the optimization problem is $\mathcal{S}(\hat{p}\|p):=\max_{\phi \in \mathcal{F}}\left\{\mathbb{E}_{x\sim \hat{p}}[\mathcal{S}_p\phi(x)]\right\}$. 
The problem is how to choose an appropriate $\mathcal{F}$. Liu et al. \cite{liu2016stein} showed that when $\mathcal{F}$ is the unit ball in an RKHS with kernel $k$, the optimal map can be expressed in the following way. Let $\mathcal{H}_0$ be an RKHS defined by a kernel $k(x, x')$ and $\mathcal{H}=\mathcal{H}_0\times\cdots\times\mathcal{H}_0$ be the $d\times 1$ vector-valued RKHS. We define $\mathcal{S}_p\otimes k(x,\cdot):=\nabla\ln p(x)k(x,\cdot)+\nabla_xk(x,\cdot)$, then, the optimal direction is given by $\phi_{\hat{p},p}^*(\cdot)=\mathbb{E}_{x\sim \hat{p}}[\nabla_x\ln p(x)k(x,\cdot)+\nabla_xk(x,\cdot)]$.
We iteratively update particles following the above direction and obtain the empirical approximation with $\{x_n\}_{n=1}^N$. Theoretical analysis has been conducted in terms of the gradient flow and has shown convergence to the true target distribution asymptotically \cite{liu2017stein}. However, no finite sample bound has been established. The norm of the optimal direction, $\mathcal{S}(\hat{p}\|p)=\|\phi_{\hat{p},p}^*\|_{\mathcal{H}}=\sqrt{\mathbb{E}_{x,y\sim \hat{p}}k_s(x,y)}$ where $k_s(x,y)=\nabla_x\nabla_yk(x,y)+\nabla_xk(x,y)\nabla_y\ln p(y)+\nabla_yk(x,y)\nabla_x\ln p(x)+k(x,y)\nabla_x\ln p(x)\nabla_y\ln p(y)$, is called kernelized stein discrepancy (KSD)\cite{liu2016kernelized}.

\subsection{Stein points(SP)}
SP \cite{chen2018stein} minimizes the above KSD directly. When $q$ is given by a discrete approximation $\hat{p}=\sum_{n=1}^N \delta(x, x_n)/N$, KSD can be written as $\mathcal{S}(\hat{p}\|p)=\sqrt{\sum_{i,j=1}^Nk_s(x_i,x_j)}$. 
In SP, to obtain the $n$-th particle, we solve $\displaystyle\argmin_{x} \sum_{i=1}^{n-1}k_s(x_i,x)$ or  $\displaystyle\argmin_{x} \sum_{i=1}^{n-1}k_s(x_i,x)+k_s(x,x)/2$. To solve these problems, the authors of the paper \cite{chen2018stein} proposed using sampling methods or grid search. However, those methods are not applicable to high-dimensional problems due to the curse of dimensionality. Although an alternative way is to use gradient descent, this is computationally difficult in high-dimensional problems since this method needs to calculate the Hessian at each iteration. Moreover, the computation cost for evaluating the derivative of the log probability is 4 times compared to SVGD. An advantage of this method is that a finite sample convergence bound is assured theoretically.

\subsection{Maximum mean discrepancy}
SVGD and SP use KSD as the direction of the update and the objective function. In our proposed method, we use MMD as the objective function. MMD is a kind of the worst-case error between expectations. For a given test function $f$, we express the integral with respect to the true posterior distribution $p$ as $Z_{f,p}=\int f(x) p(x) dx$.
We denote an approximation of $Z_{f,p}$ as $Z_{f,\hat{p}}$, where $p$ is approximated by $\hat{p}$ in the same way as Eq.\,(1).
From here, we consider the weighted empirical distribution $\hat{p}(x)=\sum_{n=1}^N w_n\delta(x, x_n)$, where $w_n$ are the weights of each particle.
Then MMD \cite{gretton2012kernel} is defined as
\begin{align}\label{mmd_def}\small
&\mathrm{MMD}(\{w_i,x_i\}_{i=1}^N)^2 :=\frac{1}{2}\underset{f\in \mathcal{H}:\|f\|_\mathcal{H}=1}{\mathrm{sup}}\left|Z_{f,p}-\sum_{i=1}^Nw_if(x_i)\right|^2\nonumber \\
&=\frac{1}{2}\|\mu_p-\mu_{\hat{p}}\|_\mathcal{H}^2 \nonumber \\
&=\frac{1}{2}\left|\left|\mu_p-\sum_{i=1}^Nw_ik(x_i,\cdot)\right|\right|_\mathcal{H}^2,
\end{align}
where $\mu_p=\int k(\cdot,x)p(x)dx \in \mathcal{H}$ and we introduce the coefficient $\frac{1}{2}$ for convenience in later calculation. We also express $\mathrm{MMD}(\{w_i,x_i\}_{i=1}^N)^2$ as $\mathrm{MMD}(\mu_{\hat{p}})^2$ for simplicity.

\section{Proposed methods}
In this section, we formally develop our MMD-FW. We will introduce the FW algorithm in an RKHS, propose our MMD-FW, and give a finite sample convergence bound of our method.
\subsection{MMD minimization by the FW algorithm (MMD-FW)}
On the basis of the existing methods reviewed in Section 2, we would like to obtain a method to approximate the posterior by discrete particles, which has high computational efficiency and theoretical guarantee. The key idea is to perform discrete approximation by minimizing MMD, instead of KSD since it causes computational problems as we mentioned in the previous section. We minimize $\mathrm{MMD}(\mu_{\hat{p}})^2=\frac{1}{2}\|\mu_p-\mu_{\hat{p}}\|_\mathcal{H}^2$, introduced by Eq. \,(\ref{mmd_def}), in a greedy way. Since this is a convex function in an RKHS, we can use the FW algorithm. 

The FW algorithm, also known as the conditional gradient method \cite{jaggi2013revisiting}, is a convex optimization method. It focuses on the problem $\displaystyle\min_{x\in \mathcal{D}}f(x)$, where $f$ is a convex and continuous differentiable function and $\mathcal{D}$ is the domain of the problem, which is also convex. As the procedure is shown in Alg.~\ref{alg:FW}, the FW algorithm optimizes the objective in a greedy way.
In each step, we solve the linearization of the original $f$ at the current state $\bm{x}_n$ as shown in Line $3$ of Alg.~\ref{alg:FW}. This step is often called the linear minimization oracle (LMO). The new state $\bm{x}_{n+1}$ is obtained by a convex combination of the previous state $\bm{x}_n$ and the solution of the LMO, $\bm{s}$, in Line $6$ of Alg.~\ref{alg:FW}. The common choice of the coefficient of the convex combination is the constant step or the line search. 
\begin{algorithm}[tb]
   \caption{Frank-Wolfe (FW) Algorithm}
   \label{alg:FW}
\begin{algorithmic}[1]
   \STATE Let $\bm{x}_0 \in \mathcal{D}$
   \FOR{$n=0 ,\ldots, N$}
   \STATE Compute $\bm{s}=\mathrm{argmin}_{\bm{s}\in \mathcal{D}}\langle \bm{s},\nabla f(\bm{x}_n)\rangle$
   \STATE Constant step: $\lambda _n=\frac{1}{n+1}$
   \STATE [Instead of constant step, use line search: $\lambda _n=\mathrm{argmin}_{\lambda \in [0,1]} f((1-\lambda)\bm{x}_{n}+\lambda \bm{s})$]
   \STATE Update $\bm{x}_{n+1}=(1-\lambda _n)\bm{x}_n+\lambda _n \bm{s}$
   \ENDFOR
\end{algorithmic}
\end{algorithm}

Bach et al. \cite{bach_herding_equi} and Briol et al. \cite{FWBQ} clarified the equivalence between kernel herding \cite{chen2010super} and the FW algorithm for MMD. In our situation, we minimize MMD on the marginal polytope $\mathcal{M}$ of the RKHS  $\mathcal{H}$, which is defined as the closure of the convex hull of $k(\cdot,x)$. We also assume that all sample points $x_i$ are uniformly bounded in the RKHS, i.e., for any sample point $x_i$, $\exists r>0 : \|k(\cdot,x)\|_{\mathcal{H}}\leq r$.

By applying the FW algorithm, we want to obtain $\mu_{\hat{p}}$ which minimizes the objective $\mathrm{MMD}(\mu_{\hat{p}})^2=\frac{1}{2}\|\mu_p-\mu_{\hat{p}}\|_\mathcal{H}^2$. We express the solution after $n$-steps FW algorithm as $\mu_{\hat{p}}^n=\sum_{i=1}^nw_i^nk(x,x_i)$, where $\{x_i\}_{i=1}^n$ are the particles and $w_i^{n}$ denote the weights of the $i$-th particle at the $n$-th iteration. We can obtain $\{x_i\}_{i=1}^n$ in a greedy way by the FW algorithm. The method of deriving the weights are discussed later.

The LMO calculation in each step is  $\mathrm{argmin}_{g \in \mathcal{M}}\langle \mu_{\hat{p}}^{n}-\mu_p,g \rangle$. It is known that the minimizer of a linear function in a convex set is one of the extreme points of the domain \cite{bach_herding_equi}, and thus we derive
\footnotesize
\begin{align}\label{LMO_herding_extreme}
\argmin_{g \in \mathcal{M}}\langle \mu_{\hat{p}}^{n}-\mu_p,g \rangle &=\argmin_{x} \langle \mu_{\hat{p}}^{n}-\mu_p,k(\cdot,x) \rangle \nonumber \\
&=\argmin_{x}\sum_{i=1}^{n}w_i^{n}k(x_i,x)-\mu_p(x).
\end{align}
\normalsize
We solve this LMO by gradient descent. We initialize each $x$ to prepare $g=k(\cdot,x)$ in LMO by sampling it from the prior distribution. Since the objective of LMO is non-convex, we cannot obtain the global optimum by gradient descent in general. Fortunately, even if we solve LMO approximately, FW enables us to establish a finite sample convergence bound \cite{locatello2017boosting, jaggi2013revisiting,lacoste2013block,lacoste2015global,locatello2017unified}.
In such an approximate LMO, we set the accuracy parameter $\delta \in (0,1]$ and consider the following approximate problem which returns approximate minimizer $\tilde{g}$ of Eq.(\ref{LMO_herding_extreme}) instead of the original strict LMO:
\begin{align}\label{approx-lmo}
\langle \mu_{\hat{p}}^{(n)}-\mu_p,\tilde{g} \rangle&=\delta \mathrm{min}_{g \in \mathcal{M}}\langle \mu_{\hat{p}}^{(n)}-\mu_p,g \rangle \nonumber \\
&=\delta \mathrm{min}_{x}\sum_{i=1}^{n}w_i^{n}k(x_i,x)-\mu_p(x).
\end{align}
This kind of relaxation of the LMO has been widely used and shown to be reliable \cite{locatello2017boosting, jaggi2013revisiting,lacoste2013block,lacoste2015global,locatello2017unified},
which is much easier to solve than the original strict LMO. We call this step Approx-LMO, and we will use gradient descent to solve Approx-LMO. The derivative with respect to $x$ when we use the symmetric kernel $k$ can be written as follows:
\footnotesize
\begin{align}\label{grad}
&\nabla_{x}\langle \mu_{\hat{p}}^{(n)}-\mu_p,g \rangle \nonumber \\
&\approx\frac{1}{n}\sum_{i=1}^{n}w_i^{(n)}\left(\nabla_{x}k(x_i,x)+k(x,x_i)\nabla_{x_i}\ln p(x_i)\right).
\end{align}
\normalsize
The derivation of Eq.(\ref{grad}) is given in Appendix. Using this gradient, we solve Eq.(\ref{approx-lmo}). As repeatedly pointed out \cite{locatello2017boosting, jaggi2013revisiting,lacoste2013block,lacoste2015global,locatello2017unified}, an approximate solution of the LMO is enough to assure the convergence which we describe later. For this reason, we will use gradient descent in our algorithm and also a rough estimate of the gradient is enough in our situation. A similar technique has also been discussed in \cite{locatello2017boosting}.

For the FW algorithm, we have to specify the initial particle $x_1$ and the step size choice of the algorithm. 
We found that the initial particle $x_1$ by the MAP estimation or approximate MAP estimation shows good performance empirically and it is recommended to prepare $x_1$ as a near MAP point (we will discuss other choices later). In this approach, the constant step size and line search are not recommended because those methods uniformly reduce the weights of all the particles which has already been obtained. When we use $x_1$ as a near MAP point, it is located near the highest probability mass regions, and thus we should not reduce its weight uniformly.
Based on this observation, we set the step size in the same way as the fully corrective Frank-Wolfe algorithm \cite{lacoste2015global}, this method calculates all the weights at each iteration, and we can circumvent the above problem. For full correction, we use the Bayesian quadrature (BQ) weight \cite{huszar2012optimally}, $w_i=\sum_m z_m K_{im}^{-1}$, where $K$ is the Gram matrix, $z_m=\int k(x, x_m)p(x)dx$, and we approximately compute the integral with particles. Since we use the empirical approximation, this makes the convergence rate slower. We will analyze the effect of this inexact step size later.

To summarize, our proposed algorithms are given in Alg.~\ref{alg:stoc_LMO} and Alg.~\ref{alg:FW_vanilla}, which greedily increase the number of particles whithin the FW framework to minimize MMD. 
\begin{algorithm}[tb]\small
   \caption{Approx-LMO}
   \label{alg:stoc_LMO}
\begin{algorithmic}[1]
   \STATE {\bfseries Input:} $\mu_{\hat{p}}^{(n)}$
   \STATE {\bfseries Output:} $k(\cdot,x^{L+1})$
   \STATE Prepare $g^0=k(\cdot,x^0)$ where $x$ is initialized by \\randomly or sample from prior
   \FOR{$l=0 \ldots L$}
   \STATE Compute $\nabla_{x}\langle \mu_{\hat{p}}^{(n)}-\mu_p,g^l \rangle$ by Eq.(\ref{grad})
   \STATE Update $x^{(l+1)}\leftarrow x^{(l)}+\epsilon^{(l)}\cdot\nabla_{x}\langle \mu_{\hat{p}}^{(n)}-\mu_p,g^i \rangle$
   \ENDFOR
\end{algorithmic}
\end{algorithm}
\begin{algorithm}[tb]\small
   \caption{{\small MMD minimization by Frank-Wolfe algorithm (MMD-FW)}}
   \label{alg:FW_vanilla}
\begin{algorithmic}[1]
   \STATE {\bfseries Input:} A target density $p(x)$
   \STATE {\bfseries Output:} A set of particles $(\{w_i,x_i\}_{i=1}^N)$
   \STATE Calculate approximate MAP estimation for $\mu_{\hat{p}}^{(1)}$
   \FOR{$n=2 \ldots N$}
   \STATE $k(\cdot, x_n)=$Approx-LMO($\mu_{\hat{p}}^{(n-1)}$)
   \STATE Empirical BQ weight: $\hat{w}_i^n=\sum_{m=1}^n \hat{z}_m K_{im}^{-1}, \hat{z}_m=\sum_{l=1}^n k(x_l, x_m)/n$
   \STATE Update $\mu_{\hat{p}}^{(n+1)}=\sum_{i=1}^{n}\hat{w}_i^{n}k(x, x_i)$
   \ENDFOR
\end{algorithmic}
\end{algorithm}

\subsection{Theoretical guarantee}
First, we describe the condition to limit the deviation of empirically approximated BQ weights from the true ones so that the condition described below are satisfied.
This is necessary for the theoretical guarantee of particle approximation.

\begin{thm}$\mathrm{(Approximate\ step\ size)}$\label{thm:step_size} 
In Alg.~\ref{alg:FW_vanilla} at the $n$-th iteration, let $\beta_i^n$ be the ratio between $\hat{z}_i^n$ and $z_i^n$, i.e.,  $\beta_i^n=\hat{z}_i^n/z_i^n$. When $\mathcal{H}$ is finite dimensional, if 
\begin{align}\label{condition}
\int k(x,y)p(x)p(y)dxdy- \sum_{i,j=1}^n\beta_i^n\beta_j^nz_i^nK^{-1}_{ij}z_j^n>0
\end{align}
holds, then Theorems \ref{thm:bound} and \ref{thm:contraction} hold. When  $\mathcal{H}$ is infinite dimensional, no condition about the deviation of the weight is needed for Theorems 2 and 3 to hold.
\end{thm}
In Eq.(\ref{condition}), since $\int k(x,y)p(x)p(y)dxdy$ is fixed and $\int k(x,y)p(x)p(y)dxdy- \sum_{i,j=1}^nz_i^nK^{-1}_{ij}z_j^n>0$, $\beta_i^n$ should be in some {\it moderate} range to satisfy the condition of Eq.(\ref{condition}).
More intuitively, this condition states that if the deviation of the empirical estimate of BQ weights from the true ones is below a certain criterion, then convergence guarantee of the algorithm still holds even if the step size is inexact. The proof is given in Appendix.
We also analyzed the effect of inexact step size in line search; see Appendix for details.

Next, we state the theoretical guarantee of our algorithm. We obtain $\hat{p}(x)=\sum_{n=1}^N w_n\delta(x, x_n)$ by Alg.~\ref{alg:FW_vanilla} which approximates the true posterior $p(x)$. Let  $f$ be the test function, then we can bound the error $|Z_{f,p}-Z_{f,\hat{p}}|=|\int f(x)p(x)dx-\sum_{i=1}^Nw_if(x_i)|$ as follows:
\begin{thm}$\mathrm{(Consistency)}$\label{thm:bound} Under the condition of Theorem \ref{thm:step_size}, the error $|Z_{f,p}-Z_{f,\hat{p}}|$ of Alg.~\ref{alg:FW_vanilla} is bounded at the following rate:
\begin{align}
&|Z_{f,p}-Z_{f,\hat{p}}|\nonumber \\
&\leq\mathrm{MMD}(\{(w_n, x_n)\}_{n=1}^N) \nonumber \\
&\leq \begin{cases}
      \sqrt{2}r e^{-\delta_{BQ}\frac{R^2\delta ^2N}{2r^2}}  & \text{if $\mathcal{H}$ is finite dimensional,} \\
      \sqrt{\frac{(\delta_{BQ}\delta+1)2^2r^2}{\delta(N\delta_{BQ}\delta+2)}} & \text{if $\mathcal{H}$ is infinite dimensional,}\\
    \end{cases}
\end{align}
where $r$ is the diameter of the marginal polytope $\mathcal{M}$, $\delta$ is the accuracy parameter of the LMO, and $R$ is the radius of the smallest ball centered at $\mu_p$ included $\mathcal{M}$ ($R$ is strictly above 0 only when the dimension of $\mathcal{H}$ is finite). $\delta_{BQ}$ denote the error caused by the empirical approximation of the BQ weights; for details, please see Appendix.
\end{thm}

A proof of Theorem 2 can be found in Appendix. Moreover, on the basis of the Bayesian quadrature, we can regard $Z_{f,\hat{p}}$ as the posterior distribution  of the Gaussian process \cite{huszar2012optimally} (see Appendix for details) and assure the posterior contraction rate \cite{FWBQ}. 
Intuitively, the posterior contraction rate indicates how fast the probability of the estimated parameter residing outside a specified region (which includes the true parameter) decreases when the size of the region is increased.
\begin{thm}$\mathrm{(Contraction)}$\label{thm:contraction} Let $S\subseteq \mathbb{R}$ be an open neighborhood of the true integral $Z_{f,p}$ and let $\gamma=\mathrm{inf}_{r'\in S^c}|r'-Z_{f,p}|>0$. Then the posterior probability on $S^c=\mathbb{R}\setminus S$ vanishes at the following rate:
\begin{align}
&\mathrm{prob}(S^c)\nonumber \\
&\leq \begin{cases}
      \frac{2r}{\sqrt{\pi}\gamma}e^{-\delta_{BQ}\frac{R^2\delta^2N}{2r^2}-\frac{\gamma^2}{4r^2}e^{\delta_{BQ}\frac{R^2\delta^2N}{r^2}}} \\ \hspace{50pt} \text{if $\mathcal{H}$ is finite dimensional,} \\
      \sqrt{\frac{2}{\pi}} \sqrt{\frac{(\delta_{BQ}\delta+1)2^2r^2}{\delta(N\delta_{BQ}\delta+2)}}e^{-\frac{\gamma^2}{2}\frac{\delta(N\delta_{BQ}\delta+2)}{(\delta_{BQ}+\delta)r^2r^2}} \\ \hspace{50pt} \text{if $\mathcal{H}$ is infinite dimensional,}\\
    \end{cases}
\end{align}
where $r$ is the diameter of the marginal polytope $\mathcal{M}$, $\delta$ is the accuracy parameter, and $R$ is the radius of the smallest ball centered at $\mu_p$ that includes $\mathcal{M}$. $\delta_{BQ}$ denotes the error caused by the empirical approximation of the BQ weights; for details, please see Appendix.
\end{thm}

In the proposed method, kernel selection is crucial both numerically and theoretically. In the above convergence proof, linear convergence occurs only under the assumption that there exists a ball with centered at $\mu_p$ whose radius $R$ is positive within the affine hull $\mathcal{M}$. Bach et al. \cite{bach_herding_equi} proved that, for infinite dimensional RKHSs, such as the case of radial basis function (RBF) kernels, such an assumption never holds. Thus, we can only have sub-linear convergence for RBF kernels in general. However, as pointed out by Briol et al. \cite{FWBQ} , even if we use RBF kernels, thanks to finite-precision rounding error in computers, we are treating in simulations are actually essentially finite dimensional. This also holds in our situation, and in experiments, we empirically observed the linear convergence of our algorithm. We will show such a numerical result later.

A theory for the constant step size and line search is shown in Appendix.

\subsection{Discussion}
For specifying the initial particle $x_1$, we can sample it from the prior distribution. The merit of this approach is that we can choose the step size in a computationally less demanding way such as the constant step size and line search (shown in Appendix) since the initial particle is not in a high probability mass region, uniformly decreasing less important weights by constant step size or line search. However, we empirically found in our preliminary experiments that this initialization does not perform well compared to MAP initialization. We suspect that the gradient of Eq.(\ref{grad}) is too inexact when initial particles are sampled from the prior.

Let us analyze the reason why MAP initialization performs well as follows.  Although the gradient is incorrect, the LMO can be solved with error to some extent because the first particle is close to the MAP estimation and the evaluation points of the expectation include, at least, a high density region on $p(x)$. If the LMO is $\delta$-close to the true value, the weights of old incorrect particles will be updated to be small enough to be ignored as the algorithm proceeds. For such a reason, the framework using processed particles works.

The empirical approximation of the BQ weights can also be justified almost in the same way as above. Since the empirical distribution includes, at least, a high density region on $p(x)$, the deviation of the step size (e.g., error due to the empirical approximation) from the exact BQ weight is smaller than the criterion in Theorem \ref{thm:step_size}.

In summary, since we prepare the initial particles at a high probability mass region, the FW algorithm successfully finds the next particle even though the gradient for LMO or weights are inexact. As the algorithm proceeds, the weights of less reliable particles become small and accuracy of the estimation is increased. This is an intuition how the proposed algorithm works.

\section{Related works}
In this section, we discuss the relationship between our method and SVGD, SP and variational boosting. 
\subsection{Relation to SVGD}
SVGD is a method of optimizing a fixed number of particles simultaneously. On the other hand, MMD-FW is a greedy method adding new particles one per step. Both methods can work in high-dimensional problems since they use the information of the gradient of the score function. To approximate a high-dimensional target distribution, we may need many particles, but it is unclear how many particles are needed beforehand. Thus, a greedy approach is preferable for high-dimensional problems. Since in SVGD it is unclear how we can increase the number of particles after we finish the optimization, MMD-FW is more convenient in such a case. However, simultaneous optimization is sometimes computationally more efficient and show better performance compared to a greedy approach(See the experimental results).

Based on this fact, we combine SVGD and MMD-FW by focusing on the fact that the update equations of SVGD and MMD-FW are almost the same except for the weights. 
More specifically, we prepare particles by SVGD first, and then apply MMD-FW by treating particles obtained by SVGD as the initial state of each greedy particle. For details, please see Appendix. This combination enables us to enjoy the efficient simultaneous optimization of SVGD and the greedy property and theoretical guarantee of MMD-FW.

In terms of computation costs, SVGD is $\mathcal{O}(N^2)$ per iteration. In MMD-FW, we only optimize one particle, and thus, its computation cost is $\mathcal{O}(N)$ at each step inside Approx-LMO . Up to the $N$-th particle, the total cost is $\mathcal{O}(N(N+1)/2)$, which is in the same order as SVGD.  However, the number of LMO iterations in MMD-FW is much smaller than that of SVGD since the problem involves only one particle in MMD-FW, which is much easier to solve than SVGD which treats $N$ particles simultaneously. Therefore, we can expect the computation cost of MMD-FW to be cheaper than SVGD.

\subsection{Relation to SP}
The biggest difference between MMD-FW and SP is the objective function. Due to this difference, we use gradient descent to obtain new particles which is still computationally effective in high-dimensional problems. However, SP minimizes KSD, so we cannot use gradient descent since the calculation requires evaluations of the Hessian at each step, which is impossible in high-dimensional problems. To cope with this problem, SP uses sampling or grid search for optimization, which does not work in high-dimensional problems due to the curse of dimensionality. As we will see later, SP does not work well with complex models such as a Bayesian neural net.

Another difference is that our method can reliably use an approximate step size for the weights of particles. We have shown how the deviation
of the approximate weights from the exact ones affects the convergence rate, which justified the use of our method even when the exact step size is unavailable.

Lastly, we use FW to establish a greedy algorithm. This enables us to utilize many useful variants of the FW algorithm such as a distributed variant\cite{wang2016parallel}. For details, see Appendix. 

However, compared with SP, we cannot evaluate the objective function directly, so we resort to other performance measures such as the log likelihood, accuracy, or RMSE in test datasets. For SP, we can directly evaluate KSD at each iteration.
\subsection{Relation to variational boosting}
The proposed method is closely related to variational boosting \cite{locatello2017boosting}. In \cite{locatello2017boosting}, the authors analyzed the variational boosting by using the FW algorithm and showed the convergence to the target distribution.
In variational boosting, a mixture of Gaussian distributions are used as an approximate posterior and its flexibility is increased the number of components in the mixture of Gaussian distributions.
An intuition behind the convergence of variational boosting is that any distribution can be expressed by appropriately combining Gaussian mixture distributions.
That situation is quite similar to MMD-FW, where we increase the number of particles greedily. In MMD-FW, we can regard each particle as being corresponding to each component of variational boosting.
In both methods, the flexibility of the approximate posterior grows as we increase the number of components or particles and this allows us to establish the linear convergence under certain conditions.
The difference is that we consider the solution in an RKHS and minimize MMD to approximate the posterior for MMD-FW, while variational boosting minimizes the KL divergence and treats the posterior in the parameter space.

\subsection{Relation to kernel herding and Bayesian quadrature}
In this paper, we are assuming that $p(x)$ is the posterior distribution. On the other hand, if $p(x)$ is a prior distribution, kernel herding \cite{chen2010super} or Bayesian quadrature \cite{ghahramani2003bayesian}, are useful.
In those methods, $x_n$'s are decided to directly minimize some criterions. For example, the kernel herding method \cite{chen2010super,bach_herding_equi} minimizes MMD in a greedy way.
The biggest difference from our method is that if $p(x)$ is the prior distribution, we can sample many particles from $p(x)$ and thus we can only choose the best particle that decreases the objective function maximally at each iteration. In MMD-FW, on the other hand, we cannot prepare the particles beforehand, and thus, we directly derive particles by gradient descent.

\subsection*{Other related work}
Recently, there has been a tendency to combine an approximation of the posterior with optimization methods, which assures us of some theoretical guarantee, e.g, \cite{locatello2017boosting,dai2016provable}. Our approach also performs discrete approximation by convex optimization in an RKHS. Another related example is sequential kernel herding\cite{lacoste2015sequential}. They applied the FW algorithm to particle filtering in state space models. While their method focused on the state space models, our proposed method is a general approximation method for Bayesian inference.

\section{Numerical experiments}

\begin{table*}[htb]\small
\centering
\caption{Benchmark results on test RMSE and log likelihood by Bayesian neural net regression model}
\vspace{0.05in}
\label{tab:bench_reg}
\resizebox{1.\textwidth}{!}{
\begin{tabular}{c|c|cc|cc|c}
\toprule
\fontsize{9}{7.2}\selectfont \bf{\multirow{2}{*}{Dataset}} & 
\fontsize{9}{7.2}\selectfont \bf{Posterior} & 
\multicolumn{2}{|c}{\fontsize{9}{7.2}\selectfont Avg. Test RMSE} & \multicolumn{2}{|c}{\fontsize{9}{7.2}\selectfont Avg. Test log likelihood} & \multicolumn{1}{|c}{\fontsize{9}{7.2}\selectfont Fixed Wall clock} \\
&\fontsize{9}{7.2}\selectfont \bf{dimension} & \fontsize{9}{7.2}\selectfont SVGD & \fontsize{9}{7.2}\selectfont Ours & \fontsize{9}{7.2}\selectfont SVGD & \fontsize{9}{7.2}\selectfont Ours &\multicolumn{1}{|c}{\fontsize{9}{7.2}\selectfont Time (Secs)} \\
\hline
\fontsize{8}{7.2}\selectfont Naval (N=11,934, D=17) &
\fontsize{8}{7.2}\selectfont 953 &
\fontsize{8}{7.2}\selectfont 4.9e-4$\pm$7.5e-5 &
\fontsize{8}{7.2}\selectfont \bf{4.2e-4$\pm$5.3e-5} &
\fontsize{8}{7.2}\selectfont $6.08\pm0.11$ &
\fontsize{8}{7.2}\selectfont \bf{6.00$\pm$0.12} &
\fontsize{8}{7.2}\selectfont 150 \\
\fontsize{8}{7.2}\selectfont Protein (N=45730, D=9) &
\fontsize{8}{7.2}\selectfont 553 &
\fontsize{8}{7.2}\selectfont $4.51\pm0.057$ &
\fontsize{8}{7.2}\selectfont \bf{4.43$\pm$0.035} &
\fontsize{8}{7.2}\selectfont $-2.93\pm0.013$ &
\fontsize{8}{7.2}\selectfont \bf{-2.91$\pm$0.0073} &
\fontsize{8}{7.2}\selectfont 40 \\
\fontsize{8}{7.2}\selectfont Year (N=515344, D=91) &
\fontsize{8}{7.2}\selectfont 9203 &
\fontsize{8}{7.2}\selectfont $9.54\pm0.08$ &
\fontsize{8}{7.2}\selectfont \bf{9.50$\pm$0.09} &
\fontsize{8}{7.2}\selectfont \bf{-3.65$\pm$0.005} &
\fontsize{8}{7.2}\selectfont \bf{-3.65$\pm$0.011} &
\fontsize{8}{7.2}\selectfont 300 \\
\bottomrule
\end{tabular}}
\end{table*}

We experimentally confirmed the usefulness of the proposed method compared with SVGD and SP in both toy datasets and real world datasets. 
Other than comparing the performance measured in terms of the accuracy or RMSE of the proposed method with SVGD and SP, we also have the following two purposes for the experiments. The first purpose of the experiments is to confirm that our algorithm is faster than SVGD in terms of wall clock time. This is because, as mentioned before in the section of relation to SVGD, it solves simple problems compared with SVGD, thus we need less number of iterations to optimize each particle than that of SVGD. The second purpose is to confirm the convergence behavior.

In all experiments, we used the radial basis function kernel, $k(x,x')=\mathrm{exp}(-\frac{1}{2h^2}|x-x'|^2)$ for proposed method and SVGD, where $h$ is the kernel bandwidth. The choice of $h$ is critical to the success of the algorithms. There are three methods to specify the bandwidth, fixed bandwidth, median trick, gradient descent.
We experimented on the above three choices and found that a fixed kernel bandwidth and the median trick are stable in general, and thus, we only show the results obtained by the median trick in this section. The results of other methods are shown with other detailed experimental settings in Appendix.
For the kernel of SP, we used the three kernels proposed by the original paper \cite{chen2018stein}: IMQ kernel $k_1(x, x') = (\alpha + ||x-x'||_2^2)^\beta$, inverse log kernel $k_2(x, x') = (\alpha + \log(1+||x-x'||_2^2))^{-1}$, and IMQ score kernel $k_3(x, x') = (\alpha + ||\nabla \log p(x) - \nabla \log p(x')||_2^2)^\beta$, where $\alpha = 1.0$ and $\beta=0.5$ are used as suggested in the original paper.

For the approx-LMO, we used Adam \cite{kingma2014adam} for all experiments. Due to space limitations, the toy data results are shown in Appendix. About the benchmark experiment, we split dataset 90$\%$ for training and 10$\%$ for testing.
\subsection*{Bayesian logistic regression}
We considered Bayesian logistic regression for binary classification. The settings were the same as in those \cite{liu2016stein}, where we put a Gaussian prior $p_0(w|\alpha)=N(0,\alpha^{-1})$ for regression weights $w$ and $p_0(\alpha)=\mathrm{Gamma}(1, 0.01)$. As the dataset, we used Covertype \cite{Dua:2017}, with 581,012 data points and 54 features. The posterior dimension is $56$. The results are shown in Fig.~\ref{classification_picture}. In Fig.~\ref{fig:cov_time_accuracy}, the vertical axis is the test accuracy and the horizontal axis is wall clock time. 
\begin{figure}[tb!]
 \centering
 \hspace{-0.in}
 \subfigure[Comparison of MMD-FW and SVGD in terms of wall clock time with the test accuracy]{
 \includegraphics[width=1.\linewidth]{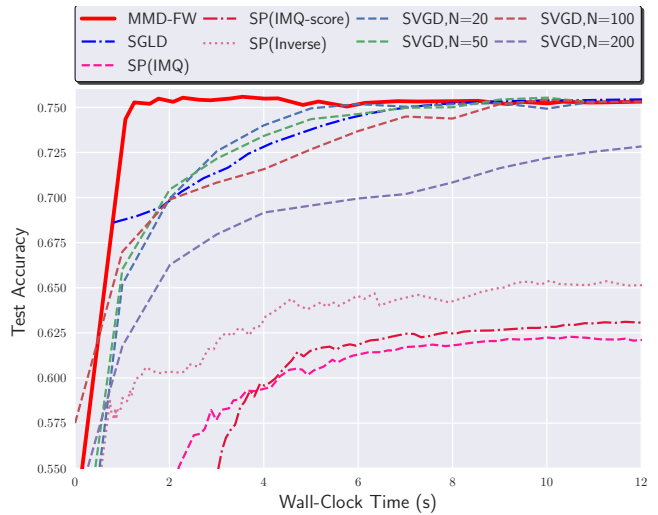}
 \label{fig:cov_time_accuracy}}
 \\\vspace{-1.mm}
 \subfigure[Convergence behavior in terms of number of the particles with $\mathrm{MMD}^2$]{
 \includegraphics[width=0.9\linewidth]{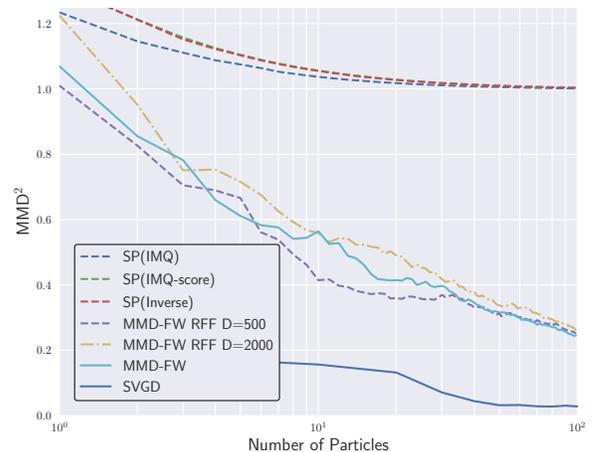}
 \label{fig:cov_mmd_fin}}
 \vspace{-2.mm}
 \caption{Comparison for the logistic regression model}
\label{classification_picture}
\end{figure}
As we discussed in Section 4.1, our algorithm was faster than SVGD in terms of wall clock time. SP did not work well. We also compared MMD-FW with stochastic gradient Langevin dynamics(SGLD) \cite{welling2011bayesian} and faster than SGLD. In Appendix, we also studied the situation where the first particle does not correspond to MAP estimation, and instead random initialization.

Fig.~\ref{fig:cov_mmd_fin} shows the convergence behavior, where the vertical axis is $\mathrm{MMD}^2$ and the horizontal one is the number of particles in the log scale. To calculate MMD, we generated ``true samples'' by Hamiltonian Monte Carlo \cite{neal2011mcmc}. 
Since RBF kernel is an infinite dimensional kernel, to further check the convergence behavior under the finite dimensional kernel, we approximated the RBF kernel by random Fourier expansion (RFF) (See Appendix for the details of the RFF). In Fig.~\ref{fig:cov_mmd_fin},  $D$ is the number of frequency of RFF. Also, we still compared with SP on MMD although this comparison is a little unfair since the objective of SP is kernelized Stein discrepancy.
As discussed in the previous section, although the convergence is sub-linear order theoretically since we used RBF kernel which is an infinite dimensional kernel, we observed the linear convergence thanks to the rounding error in the computer. The convergence speed of RBF kernel approximated by RFF showed the linear, which is the expected behavior since the approximated kernel by RFF is the finite dimensional kernel. 

SVGD had a smaller MMD than the proposed method, which is due to the fact that SVGD simultaneously optimizes all particles and tries to put particles in the best position in correspondence with the global optima. In contrast, MMD-FW only increased the particles greedily, and this resulted in local optima. Hence, the better performance of SVGD compared with MMD-FW with the same number of particles in terms of MMD is a natural result.

\begin{figure}[ht!]
 \centering
 \hspace{-0.in}
 \subfigure[Comparison of MMD-FW and SVGD in terms of wall clock time with test RMSE]{
 \includegraphics[width=.9\linewidth]{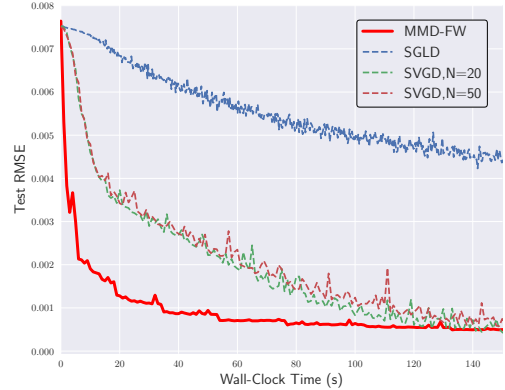}
 \label{fig:bnn_rmse_naval}}
 \\
 \vspace{-3.mm}
 \subfigure[Convergence behavior in terms of the number of the particles with $\mathrm{MMD}^2$]{
 \includegraphics[width=.9\linewidth]{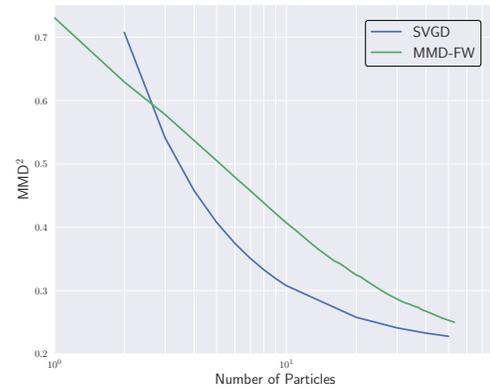}
 \label{fig:bnn_mmd_naval}}
 \vspace{-2.mm}
 \caption{Comparison for the Bayesian neural net regression model}
 \label{pic:bnn_naval_fin}
\end{figure}

\subsection*{Bayesian neural net regression}

We experimented with Bayesian neural networks for regression. The settings were the same as those in \cite{liu2016stein}. We used a neural network with one hidden layer, 50 units, and the ReLU activation function. As the dataset, we used the Naval data from the UCI \cite{Dua:2017}, which contains 11,934 data points and 17 features. The posterior dimension was 953. The results are shown in Fig.\,\ref{pic:bnn_naval_fin}. In Fig.~\ref{fig:bnn_rmse_naval}, the vertical axis is the test RMSE, and the horizontal axis is wall clock time. In Fig~\ref{fig:bnn_mmd_naval}, the vertical axis is the $\mathrm{MMD}^2$, and the horizontal axis is the number of particles. Since it is difficult to prepare MAP initialization for Bayesian neural networks at first in MMD-FW, we consider non-MAP initialization, and we gradually reduced earlier weight sizes by adjusting the step size. The posterior dimension was much higher than that of the logistic regression, but our algorithm was faster than SVGD in terms of wall clock time and linearly converged, which is consistent with the theory.

Results for other datasets are shown in Table~\ref{tab:bench_reg}, where we fixed the wall clock time and applied MMD-FW and SVGD within that period. SP did not work well because of the high dimensionality so its results are not shown. We experimented 5 random trials for changing the splitting of the dataset. For the Protein data, we used the same model as the Naval data, and for the Year data, we used the same model as others except that the number of hidden units is 100.
From these benchmark dataset experiments, we confirmed that our method shows almost the same performance as SVGD in many cases but shows faster optimization. Moreover, it shows linear convergence. 
\section{Conclusions}
In this work, we proposed MMD-FW, a novel approximation method for posterior distributions. Our method enjoys empirically good performance and theoretical guarantee simultaneously. In practice, our algorithm is faster than existing methods in terms of wall clock time and works well even in high-dimensional problems.
As future work, we will further apply this framework other than the posterior approximation and further analyze the effect of rounding error on convergence rate.

\section{Acknowledgment}
FF was supported by JST AIP-PRISM Grant Number JPMJCR18ZH Japan and
Google Fellowship,
IS was supported by KAKENHI 17H04693,
and
MS was supported by KAKENHI 17H00757.

\bibliography{stein.bib}

\appendix

\section{Proof of Eq.~(5) in the main paper}\label{sup:gradient}
For the symmetric kernel $k$, the relation $\nabla_x k(x-y)=\nabla_y k(y-x)$ holds, and we apply the partial integral method to the first term, then
\begin{align}
&\nabla_{x_n}\int k(x,x_n)p(x)dx \nonumber \\
&=\int \nabla_{x_n}k(x,x_n)p(x)dx \nonumber \\
&=\int \{\nabla_{x}k(x_n,x)\}p(x)dx\nonumber \\
&=k(x_n,x)p(x)\big|_{-\infty}^\infty-\int k(x_n,x)\nabla_{x}p(x)dx\nonumber \\
&=-\mathbb{E}_{p(x)} \left[k(x,x_n)\nabla_{x}\ln p(x)\right]
\end{align}
To approximate the integral, we usually use importance sampling when the analytic form of the integral is not available. 
Here, MMD-FW is the greedy approach, therefore we have particles which had already been processed. 
Thus we approximate the expectation by the empirical distributions which are composed of processed particles. 
The FW framework does not need the exact solution of LMO.
We just approximately solve it.
At the early stage of the algorithm, 
there are not so many particles and there might exist the unreliable particles, 
hence the expectation is not so reliable. 
Fortunately, they are enough to solve the LMO approximately. 
And as the algorithm proceeds,
the weights of those early unreliable particles are gradually reduced by the step size.
Thus, we solve the LMO by the gradient which is written by
\begin{align}
\nabla_{x_n}\int k(x,x_n)p(x)dx\simeq -\frac{1}{N}\sum_{m=1}^N k(x_m,x_n)\nabla_{x_m}\ln p(x_m).
\end{align}
Thus, we can obtain the update equation,
\begin{align}
&\nabla_{x}\langle \mu_{\hat{p}}^{(n)}-\mu_p,g \rangle \nonumber \\
&=\frac{1}{n}\sum_{l=1}^{n}w_l^{(n)}\nabla_{x}k(x_l,x)+\frac{1}{n}\sum_{l=1}^{n} w_l^{(n)}k(x,x_l)\nabla_{x_l}\ln p(x_l).
\end{align}
In the above expression, the first term corresponds to the regularization term, which try to scatter the particles. When we use the RBF kernel, the first term is proportional to the inverse of the bandwidth. Thus, it is easily understand that small bandwidth makes regularization term large, and vise versa. The second term try to move particles in high mass regions.

The justification of using the existing particle for the integral approximation is based on no need to strictly solve the LMO. Although the gradient is incorrect, the LMO can be solved with error to some extent because the first particle is close to MAP and the evaluation points of the expectation include, at least, one region with high density on p(x). If the LMO is $\delta$-close to the true value, the weights of old incorrect particles will be updated to be small enough to ignore as the algorithm proceeds. Therefore the framework using processed particles works. The key trick for this is that the initial particle is close to MAP. This kind of inexact gradient descent is widely used in FW algorithm.

\section{Step size selection}\label{sup:step_size}
\begin{algorithm}[t]
   \caption{{\small MMD minimization by Frank-Wolfe algorithm (MMD-FW)}}
   \label{alg:FW_three_steps}
\begin{algorithmic}[1]
\STATE {\bfseries Input:} A target density $p(x)$
   \STATE {\bfseries Output:} A set of particles $(\{w_i,x_i\}_{i=1}^N)$
   \STATE Calculate approximate MAP estimation for $\mu_{\hat{p}}^{(1)}$
   \FOR{$n=2 \ldots N$}
   \STATE $k(\cdot, x_n)=$Approx-LMO($\mu_{\hat{p}}^{(n-1)}$)
   \IF{Constant step}
   \STATE $\lambda _n=\frac{1}{\gamma+1}$
   \STATE Update $\mu_{\hat{p}}^{(n+1)}=(1-\lambda _l)\mu_{\hat{p}}^{(n)}+\lambda _n \bar{g}_n$
   \ELSIF{Line search: }
   \STATE $\lambda _n=\mathrm{argmin}_{\lambda \in [0,1]} J((1-\lambda)\mu_{\hat{p}}^{(n)}+\lambda \bar{g}_n)$
   \STATE Update $\mu_{\hat{p}}^{(n+1)}=(1-\lambda _l)\mu_{\hat{p}}^{(n)}+\lambda _n \bar{g}_n$
   \ELSE
   \STATE Empirical BQ weight: $\hat{w}_i^n=\sum_{m=1}^n \hat{z}_m K_{im}^{-1}, \hat{z}_m=\sum_{l=1}^n k(x_l, x_m)/n$
   \STATE Update $\mu_{\hat{p}}^{(n+1)}=\sum_{i=1}^{n}\hat{w}_i^{n}k(x, x_i)$
   \ENDIF
   \ENDFOR
\end{algorithmic}
\end{algorithm}

An appropriate step size is crucial for the success of the FW.
Generally, there are three choices as shown in Alg~\ref{alg:FW_three_steps}.
Common choices are the constant step size and Line search.
The step size of line search can be written as
\begin{align}
\lambda _n=\frac{\langle \mu_{\hat{p}}^{(n)}-\mu_p,\mu_{\hat{p}}^{(n)}-\bar{g}_{n-1} \rangle}{\|g_{i-1}-\bar{g}_{n}\|^2_\mathcal{H}}.
\end{align}
The point is that they constantly reduce the weights of earlier particles.
Thus, those weights are preferable when the early particles are not reliable.
In our algorithm, 
those weights are not preferable since we use the near MAP initialization.

Another choice of the step size is the Fully correction.
As the name means, 
this method updates the weights of all particles which have already been obtained at the previous steps.
In our algorithm,
we used the Bayesian Quadrature weights since they are the optimal weights for the MMD.
For more details of the Bayesian Quadrature,
please see the Appendix of quadrature rules.
The weights of BQ can be calculated by
\begin{align}
w_\mathrm{BQ}^{(n)}=\sum_m z_j^\top K_{nm}^{-1},
\end{align}
where $K$ is the gram matrix, $z_n=\int k(x,x_n)p(x)dx$ and we approximate the integral by particles. 
Fully correction is preferable when the early particles are important. So, this choice is preferable in our algorithm

The step size choice affects the convergence rate directly.
In the main text, we only showed the results of Fully correction. Actually, line search error bound is the same as the BQ step size. Thus here we only show the error bound of constant step and $\mathcal{H}$ is finite dimensional.
\begin{thm}$\mathrm{(Consistency)}$ Under the condition of Theorem 1 in the main paper, the error $|Z_{f,p}-Z_{f,\hat{p}}|$ of Alg.~1 is bounded at the following rate:
\begin{align}
&|Z_{f,p}-Z_{f,\hat{p}}|
\leq\mathrm{MMD}(\{(w_n, x_n)\}_{n=1}^N)
\leq \frac{2r^2}{R\delta N}
\end{align}
where $r$ is the diameter of the marginal polytope $\mathcal{M}$, $\delta$ is the accuracy parameter, and $R$ is the radius of the smallest ball of center $\mu_p$ included $\mathcal{M}$.
\end{thm}

Also, when we reweight the obtained particles by using Bayesian quadrature weights and interpret it as the posterior, the following contraction property holds (line search result is the same as the BQ, and infinite RKHS result of constant step is the same as BQ, we only show the result of constant step in finite RKHS).
\begin{thm}$\mathrm{(Contraction)}$ Let $S\subseteq \mathbb{R}$ be an open neighborhood of the true integral $Z_{f,p}$ and let $\gamma=\mathrm{inf}_{r\in S^c}|r-Z_{f,p}|>0$. Then the posterior probability of mass on $S^c=\mathbb{R}\setminus S$ by Alg~3 in the main paper vanishes at the rate:
\begin{align}
&\mathrm{prob}(S^c)\leq\frac{2\sqrt{2}r^2}{\sqrt{\pi}R\delta\gamma N}e^{-\frac{\gamma^2R^2\delta^2N^2}{8r^4}} \nonumber \\
\end{align}
where $d$ is the diameter of the marginal polytope $\mathcal{M}$, $\delta$ is the accuracy parameter, $R$ is the radius of the smallest ball of center $\mu_p$ included $\mathcal{M}$.
\end{thm}

\section{Proof of Theorem~2 in the main paper}\label{sup:bound_proof}
First, we consider the case of Line search variants. The proof goes almost in the same way as
\cite{beck2004conditional}. (The proof of \cite{guelat1986some} is also useful.) 

Here the notation is the same as \cite{beck2004conditional} and we prove our theorem in the same way as Proposition 3.2. in \cite{beck2004conditional}. We extend the Proposition 3.2 to the situation where we use the approximate LMO, where \cite{beck2004conditional} consider the case of LMO.
The problem in \cite{beck2004conditional} is
\begin{align}
v^*=\min_{x\in\mathcal{S}}\frac{1}{2}\|Mx-g\|^2.
\end{align}
and we solve this by FW.
Here $g\in\mathcal{R}^m$ and $M:\mathcal{R}^n\to\mathcal{R}^m$ is the matrix. This is the convex problem on the domain $\mathcal{S}\subset \mathcal{R}^n$. The proof in \cite{beck2004conditional} can be easily applied to the finite dimensional functional problem, where the objective is the $\mathrm{MMD}^2$. Here first we state the general strategy of proving the linear convergence of the above problem by the line search FW algorithm. If you want to see the whole proof, please check \cite{beck2004conditional}.

We consider to solve the above problem by FW algorithm. We express the solution of the linearization of the above problem as $p$, that is, if we express the initial point as $x_0\in \mathcal{M}$, and express the $k-1$-th linearization problem and its solution as $p_{k-1}:=\argmin_{p\in\mathcal{S}} \{\langle p-x_{k-1},\nabla f(x_{k-1})\rangle\}$. And if the step size $\lambda_{k-1}$ is obtained via constant step or Line search, then the next state is calculated by $x_k=x_{k-1}+\lambda_{k-1}(p_{k-1}-x_{k-1})$. We express $v_k:=g-Mx_{k-1}$, $w_k:=g-Mp_{k-1}$. 
Base on this definition, $\nabla f(x_{k-1})=M^\top(Mx_{k-1}-g)$, thus LMO problem can be written as
\begin{align}
p_{k-1}:&=\argmin_{p\in\mathcal{S}} \{\langle p-x_{k-1},M^\top(Mx_{k-1}-g)\rangle\}\nonumber \\
&=\argmin_{p\in\mathcal{S}} \{\langle M(p-x_{k-1})-g+g,Mx_{k-1}-g\rangle\}\nonumber \\
&=\argmin_{p\in\mathcal{S}} \{\langle Mp-g+v_{k-1},-v_{k-1}\rangle\}\nonumber \\
&=\argmin_{p\in\mathcal{S}} \{\langle g-Mp,v_{k-1}\rangle\}
\end{align}
Thus, the LMO problem can be characterized as
\begin{align}\label{LMO_paper}
\langle w_{k-1},v_{k-1}\rangle=\min_{p\in\mathcal{S}} \{\langle g-Mp,v_{k-1}\rangle\}
\end{align}
Also $\|v_k\|^2$ denotes the error of the algorithm at $k$-th step.

Let us consider the line search step size. By the straightforward calculation of the definition of the line search, we can show that the line search step size is
\begin{align}
\lambda _{k-1}=\frac{\langle v_{k-1},v_{k-1}-w_{k-1} \rangle}{\|v_{k-1}-w_{k-1}\|^2}.
\end{align}
if this $\lambda_{k-1}\leq 1$, since we assumed that the step size is smaller than 1. Based on this step size,we can show that
\begin{align}
\|v_{k}\|^2=\|g-Mx_k\|^2=\frac{\|v_{k-1}\|^2\|w_{k-1}\|^2-\langle v_{k-1},w_{k-1} \rangle^2}{\|v_{k-1}-w_{k-1}\|^2}.
\end{align}
From Proposition 3.1 in \cite{beck2004conditional}, following relation holds,
\begin{align}\label{sup:inner_product}
\langle v_k,w_k\rangle\leq-R_s(\hat{x},M)\|v_k\|.
\end{align}
This says that there exists a ball whose radius is $R_s(\hat{x},M)$  centerted lies within $\mathcal{M}$. By using this relation, we can show that
\begin{align}
\|v_{k}\|^2\leq \left(1-\frac{R^2}{\|w_{k-1}\|^2}\right)\|v_{k-1}\|^2.
\end{align}
Finally, since the domain $\mathcal{S}$ is the bounded set, it is contained in some larger ball whose radius is $\rho_S$ and thus, the relation $\|w_{k-1}\|\leq\|g-Mp_{k-1}\|\leq\|g\|+\|M\|\rho_s$ holds. Thus
\begin{align}
\|v_{k}\|^2\leq \left(1-\left(\frac{R}{\|g\|+M\|\rho_s\|}\right)^2\right)\|v_{k-1}\|^2.
\end{align}
holds and this means the linear convergence of the problem, since $\|v_{k}\|$ express the  error of the algorithm at iteration $k$.

Base on the original proof, let us consider the approximate LMO whose accuracy parameter is $\delta$. As we saw, the solution of the LMO problem can be written as, Eq.(\ref{LMO_paper}). Approximate LMO returns $\tilde{w}$ which deviates from the true $w$ in the following way 
\begin{align}
\langle v_k,\tilde{w}_k\rangle\leq\delta\langle v_k,w_k\rangle.
\end{align}
This is derived straightforwardly from the definition of the approximate LMO.
From this definition, following holds by Eq.(\ref{sup:inner_product})
\begin{align}\label{App:approx_LMO}
\langle v_k,\tilde{w}_k\rangle\leq-\delta R_s(\hat{x},M)\|v_k\|.
\end{align}
Here after, for simplicity, we assume that step size of the line search and BQ are obtained without approximation(In our algorithm, they are approximated by empirical approximation). Later, we will discuss the those inexact step sizes.

Based on the above approximate LMO relation, we replace the $w_k$ by $\tilde{w}_k$ in the proof of Proposition 3.2. in \cite{beck2004conditional}, and we obtain the variant of Eq.(12) in \cite{beck2004conditional} which uses approximate LMO not LMO. After this, we use Eq.(\ref{App:approx_LMO}) for the evaluation of $\|v_k\|^2$ and we can obtain the following expression,
\begin{align}
\|v_k\|^2\leq\left(1-\left(\frac{\delta R_s(\hat{x},M)}{\|g\|+\rho_s\|M\|}\right)^2\right)\|v_{k-1}\|^2
\end{align}
From this expression, we can proof the linear convergence about $v$. 

Based on this bound, we can apply the result of Ch.4.2 in \cite{bach_herding_equi}.

Then, by utilizing the discussion of Appendix B in \cite{FWBQ}, we can obtain the following expression.
\begin{align}
|Z_{f,p}-Z_{f,\hat{p}}|
\leq\mathrm{MMD}(\{(w_n, x_n)\}_{n=1}^N) \leq\|\mu_p-\mu_{\hat{p}}\|
\end{align}
This is derived Cauchy Schwartz inequality and the definition of MMD and $\|f\|_\mathcal{H}\leq1$.
Thus, we have proved the theorem in the case of line search.

Since fully corrective variants optimize all the weights, the bound of this is superior to that of the line search. Hence
\begin{align}
\|v_k^\mathrm{FC}\|^2\leq\|v_k\|^2\leq\left(1-\left(\frac{\delta R_s(\hat{x},M)}{\|g\|+\rho_s\|M\|}\right)^2\right)\|v_{k-1}\|^2
\end{align}
where $v_k^\mathrm{FC}$ is derived by fully corrective variants. Thus we can bound the fully corrective variant in the same expression as line search. 
Also, geometric convergence of fully correction variant is discussed in \cite{locatello2017boosting},\cite{lacoste2015global}. They also discussed it by using the fact that fully correction is superior to line search. So far we have worked on the problem in \cite{beck2004conditional}, but this result is directly applicable to the finite dimensional RKHS problem, see \cite{bach_herding_equi}.

Next, we consider the constant step case. 
Let us assume that the step size at iteration $k$ is $\frac{1}{k+1}$. Based on the above notations, we get
\begin{align}\label{App:approx_LMO_constant_step}
&\|Mx_{k+1}-g\|^2=\|M(\lambda_kp_k+(1=\lambda_k)x_k)-g\|^2\nonumber \\
&=\left\|\frac{1}{k+1}Mp_{k}+\frac{k}{k+1}Mx_{k}-\frac{k+1}{k+1}g\right\|^2 \nonumber \\
&=\frac{1}{(k+1)^2}\|Mp_{k}-g+k(Mx_k-g)\|^2
\end{align}
Then we set $v_k'=kv_k$, we get the following,
\begin{align}
v_{k+1}'=\|Mp_{k}-g+v_k'\|^2
\end{align}
We will bound the above expression by,
\begin{align}\label{eq:seq_ineq}
v_{k+1}'&=\|Mp_{k}-g\|^2+\|v_k'\|^2+2\langle Mp_{k}-g,v_k'\rangle \nonumber \\
&=\left(\|g\|+\rho_s\|M\|\right)^2+\|v_k'\|^2-2\delta R\|v_k'\| \nonumber \\
&=\|v_k'\|^2+\|v_k'\|\left(\frac{(\|g\|+\rho_s\|M\|)^2}{\|v_k'\|}-2\delta R\right)
\end{align}
where we used the inequality which we also used in the line search version and $\delta$ is the accuracy parameter of approximate LMO.
Here we set $C_k=\frac{(\|g\|+\rho_s\|M\|)^2}{2\delta R}$. We show that $\|v_{k}'\|\leq C_k$ by the induction. The initial inequality about $\|v_1\|$ is clear. If the condition $\|v_{k}'\|\leq C_k$ satisfied, we prepare the parameter $\alpha \in (0,1]$, which satisfies $\|v_{k}'\|=\alpha C_k$. Then we substitute $\|v_{k}'\|=\alpha C_k$ to Eq.(\ref{eq:seq_ineq}), then we get
\begin{align}\label{eq:seq_ineq}
v_{k+1}'&\leq\alpha^2C_k^2\leq C_k^2
\end{align}
Thus we showed that $\|v_{k}'\|\leq C_k$. Since $\|v_{k}'\|=k\|v_k\|$, this ends the proof.
(We can also prove the above in the same way as in \cite{chen2010super}.)

In the above proof, we consider the fixed accuracy parameter $\delta$ for the approximate LMO. However, the $\delta$ can be different at each approximate LMO calls. In that situation, we express the accuracy parameter of the $k$-th call as $\delta_k$. We consider the worst accuracy LMO call and define $\delta=\min_k\delta_k$.
About the Line search, if we put $\delta_k^2q^2=\left(\frac{R_s(\hat{x},M)}{\|g\|+\rho_s\|M\|}\right)^2$, then following relation holds,
\begin{align}
\|v_k\|^2 & \leq\|v_{0}\|^2e^{-q^2\sum_{l=0}^k\delta_k}\nonumber \\
& \leq\|v_{0}\|^2e^{-q^2k\min_k\delta_k}\nonumber \\
&=\|v_{0}\|^2e^{-q^2k\delta}
\end{align}

The above discussion depends on the existence of a ball inside the domain. Next we discuss about the infinite RKHS situation, where a ball does not exist. For the proof, we just utilize the standard FW proof. The proof is the same in  \cite{locatello2017unified}. We use the same notation in the finite RKHS case. Since we use the line search, the optimal step is
\begin{align}
\lambda^*=\min \left(\frac{\langle v_{k-1}, v_{k-1}-w_{k-1}\rangle}{\|v_{k-1}-w_{k-1}\|^2},1\right)
\end{align}
By using the following identity, that set $f(x)=\frac{1}{2}\|Mx-g\|^2$
\begin{align}
f(x+\lambda(p-x))=f(x)+\lambda\langle p-x, \nabla f(x)\rangle+\frac{\lambda^2}{2}\|M(x-p)\|^2
\end{align}
and using the relation, $\langle p_{k-1}-x_{k-1},\nabla f(x_{k-1}) \rangle=\langle v_{k-1},w_{k-1} \rangle-\|v_{k-1}\|^2\leq -\|v_{k-1}\|^2$.
Line search step size minimizes the right hand side of the above with respect to $\lambda$. From this, by using the approx-LMO with accuracy parameter $\delta$, we get the following inequality
\begin{align}
\|v_{k+1}\|^2&\leq\|v_{k}\|^2+\min_\lambda\left\{-\lambda\delta\|v_k\|^2 +\frac{\lambda^2}{2}(2\|M\|\rho_s)^2\right\} \nonumber \\
&\leq \|v_k\|^2-\frac{2}{\delta k+2}\delta\|v_k\|^2+\frac{2}{(\delta k+2)^2}(2\|M\|\rho_s)^2
\end{align}
, here we set $\lambda=\frac{2}{\delta k+2}$. Since this is not the optimal weight we get the second inequality in the above.
By the induction, we can prove
\begin{align}
\|v_k\|^2\leq2\frac{(1+\delta)(2\|M\|\rho_s)^2}{\delta(\delta k+2)}
\end{align}
, this is the same way as the standard FW algorithm.
This ends the proof.

\section{Analyzing the inexact step size}
In this section, we analyze the effect of inexact step size on the convergence rate.

First, we will see the step size of line search in finite dimensional case.
The calculation of the step size in line search includes the $\langle \mu_p,g\rangle=\int k(x,x')dx'$ which is intractable in general if $p(x)$ is posterior distribution. For the analysis, we express the exactly calculated step size of line search by $\lambda$ and $\lambda'$ denotes the step size in which the above integration is approximated by empirical distribution. We also express the ratio of $\lambda$ and $\lambda'$ as $\alpha$, where $\lambda'=\alpha \lambda$. This $\alpha$ express the deviation from the exact step size $\lambda$. We analyze what range of $\alpha$ is still enough to assure the exponential convergence in finite dimensional kernel
\begin{thm}(Inexact step size in for line search)\label{thm:line_step_size} 
If the ratio $\alpha$ is bounded inside $(0,2)$, then exponential convergence still holds.
\end{thm}
\begin{proof}
First, let us go back to the proof of exponential convergence for line search step size. Since we express the approximated step size by $\lambda'=\alpha\lambda$,
\begin{align}
&\|v_{k+1}^2\|\nonumber \\
&=\|g-Mx^{k+1}\|^2 \nonumber \\
&=\lambda'^2\|v_k-w_k\|^2+2\lambda'\langle v_{k}, w_{k}-v_k\rangle+\|v_k\|^2 \nonumber \\
&=(1-\alpha)^2\|v_k\|^2-2(1-\alpha)^2\langle v_{k}, w_{k}\rangle \nonumber \\
& +(\alpha^2-2\alpha)\langle v_{k}, w_{k}\rangle^2+\|v_k\|^2\|w_k\|^2
\end{align}
From this, we can bound the right hand side in the same way as before, but which includes the additional coefficients
\begin{align}
&\|v_{k+1}^2\|\leq\left\{1-\frac{\alpha(2-\alpha)R^2}{(\|g\|+\rho_s\|M\|)^2}\right\}\|v_{k}^2\|
\end{align}
Thus, to enhance the geometrical decrease, $\alpha(2-\alpha)>0$ is needed. This ends the proof.
\end{proof}

\subsubsection{Inexact BQ weights}
Next, we analyze the approximate BQ weights. To do that, we briefly review the BQ.

In the Bayesian Quadrature method\cite{ghahramani2003bayesian,huszar2012optimally}, we put on the Gaussian process prior on $f$ with kernel $k$ and mean $0$. In usual gaussian process, after conditioned on $f(X)=\left(f(x_1), \dots,f(x_N)\right)^\top$, we can obtain the closed-form posterior distribution of $f$,
\begin{align}\label{GP}
p(f(x_*)|p(f(X)))=N(f(x_*)|\mu,\Sigma),
\end{align}
where $\mu=k(x_*,X)K^{-1}f(X)$, $\Sigma=k(x_*,x_*)-k(x_*,X)K^{-1}k(X,x_*)$, here $K_{i,j}=K(x_i,x_j)$ and $N(x|\mu,\Sigma)$ means the Gaussian distribution with mean $\mu$ and the covariance $\Sigma$. Thanks to the property of Gaussian process that linear projection preserves the normality, the integrand is also Gaussian, and thus we can obtain the posterior distribution of the integrand as follows,
\begin{align}\label{Z_GP}
\mathbb{E}_\mathrm{GP}[Z_{f,p}]&=\mathbb{E}_\mathrm{GP}\left[\int f(x)p(x)dx\right]\nonumber \\
&=\iint f(x)p\left(f(x)|p(f(X))\right)p(x)dxdf\nonumber \\
&=\int k(x,X)K^{-1}f(X)p(x)dx \nonumber \\
&=\bm{z}^\top K^{-1}f(X)
\end{align}
where $z_n=\int k(x,x_n)p(x)dx$. From the above expression,
\begin{align}\label{weight_BQ}
\mathbb{E}_\mathrm{GP}[Z_{f,p}]=\sum_{n=1}^N w_\mathrm{BQ}^{(n)}f(x_n),\ \ w_\mathrm{BQ}^{(n)}=\sum_m z_j^\top K_{nm}^{-1}.
\end{align}
In the same way as the expectation, we can calculate the variance of the posterior,
\begin{align}\label{original_def_mmd_var}
\mathbb{V}[Z_{f,p}|f(x_1),\dots f(x_N)]=\iint k(x,x')dxdx'-\bm{z}^\top K^{-1}\bm{z}
\end{align}
\cite{huszar2012optimally} proved that in the RKHS setting, minimizing the posterior variance corresponds to minimizing the MMD,
\begin{align}
\mathbb{V}[Z_{f,p}|f(x_1),\dots f(x_N)]=\mathrm{MMD}^2(\{(w_\mathrm{BQ}^{(n)},x_n)\}_{n=1}^N).
\end{align}
The BQ minimize the above discrepancy greedily in the following way,
\begin{align}\label{mmd_greedily}
x_{N+1} \leftarrow &\argmin_x \mathbb{V}[Z_{f,p}|f(x_1),\dots f(x_N), f(x)].
\end{align}
\cite{huszar2012optimally} showed that 
\begin{align}\label{optimal}
\mathrm{MMD}(\{(w_\mathrm{BQ}^{(n)},x_n)\}_{n=1}^N) = \underset{w \in \mathbb{R}^N}{\mathrm{inf}} \underset{f\in \mathcal{H}:\|f\|_\mathcal{H}=1}{\mathrm{sup}}|Z_{f,p}-\hat{Z}_{f,p}|
\end{align}
and thus,
\begin{align}\label{optimal_BQ}
\mathrm{MMD}(\{(w_\mathrm{BQ}^{(n)},x_n)\}_{n=1}^N) \leq \mathrm{MMD}(\{(w_n=\frac{1}{N},x_n)\}_{n=1}^N)
\end{align}
Now we analyze how $\mathrm{MMD}(\{(w_\mathrm{BQ}^{(n)},x_n)\}_{n=1}^{k+1})^2$ and $\mathrm{MMD}(\{(w_\mathrm{BQ}^{(n)},x_n)\}_{n=1}^{k})^2$ differs. This is explicitly calculated by Eq.(\ref{original_def_mmd_var}),
\begin{align}\label{quadratic}
\mathrm{MMD}(\{(w_\mathrm{BQ}^{(n)},x_n)\}_{n=1}^{k+1})^2- \mathrm{MMD}(\{(w_\mathrm{BQ}^{(n)},x_n)\}_{n=1}^{k})^2\nonumber \\
=-\bm{z}_{(k+1)}^\top K_{(k+1)}^{-1}\bm{z}_{(k+1)}+\bm{z}_{(k)}^\top K_{(k)}^{-1}\bm{z}_{(k)}
\end{align}
where $K_{(k)}$ denotes the Gram matrix using data $x_1$ to $x_k$ and $\bm{z}_{(k)}=(\int k(x_1,x)dx,\dots,\int k(x_k,x)dx)^\top$. Since this quantity is the difference of quadratic form, it is convenient for the analysis based on their eigenvalues. Here we assume that $K_{(k)}$ and $K_{(k+1)}$ are full rank. Since they are gram matrix of positive definite kernel, there exists different positive $k$ eigenvalues for the  matrix $K_{(k)}$. We denote those eigenvalues by $\gamma_i, i=1\dots k$, and let $e_i$ be its eigenvector, $K_{(k)}e_i=\gamma e_i$. Let $U=(e_1,\dots,e_k)$, then by diagonalization 
\begin{align}
K_{(k)}&=U\left(
    \begin{array}{ccc}
      \gamma_1 & \dots & 0 \\
       & \ddots &  \\
      0 & \dots & \gamma_k  
    \end{array}
  \right)U^\top \\
&=U\Gamma U^\top
\end{align}
From the inverse matrix property,
\begin{align}
K_{(k)}^{-1}&=U\left(
    \begin{array}{ccc}
      \gamma_1^{-1} & \dots & 0 \\
       & \ddots &  \\
      0 & \dots & \gamma_k^{-1}  
    \end{array}
  \right)U^\top \\
&=U\Gamma^{-1} U^\top.
\end{align}

By diagonalization, 
\begin{align}
\mathrm{MMD}(\{(w_\mathrm{BQ}^{(n)},x_n)\}_{n=1}^{k})^2=\sum_{i=1}^k\gamma_i^{-1}z_i'^\top z_i'
\end{align}
where $z_i'=U^\top z_i$.
Next, about $K_{(k+1)}$, we investigate its eigenvalues. We can express $K_{(k+1)}$,
\begin{align}\label{gram_k}
K_{(k+1)}=\left(
    \begin{array}{cc}
      K_{(k)} & \tilde{k}_{(k+1)}  \\
      \tilde{k}_{(k+1)}^\top & 1 
    \end{array}
  \right)
\end{align}
where $\tilde{k}_{k+1}^\top=(k(x_{k+1},x_1)\dots k(x_{k+1},x_k))^\top$. Let $E_{k}$ be the $k\times k$ identity matrix. Then the eigenvalue of $K_{(k+1)}$ can be calculated by solving the following equation.
\begin{align}
0=\mathrm{det}\left(
    \begin{array}{cc}
      K_{(k)}-\gamma^*E_k & \tilde{k}_{(k+1)}  \\
      \tilde{k}_{(k+1)}^\top & 1 -\gamma^*
    \end{array}
  \right)
\end{align}
We use the determinant formula,
\begin{align}
  \mathrm{det} \left(
    \begin{array}{cc}
      A & B  \\
      C & D 
    \end{array}
  \right)=\mathrm{det}A\ \mathrm{det}(D-CA^{-1}B)
\end{align}
where regularity is assumed for $A$ and $D$. Then,
\begin{align}
0=\mathrm{det}(K_{(k)}-\gamma^*E_k)\left((1 -\gamma^*)-\tilde{k}_{n+1}^\top(K_{(k)}-\gamma^*E_k)^{-1}\tilde{k}_{n+1}\right)
\end{align}
From the first term, we can see that $K_{(k+1)}$ have $(\gamma_1,\dots,\gamma_k)$ as the eigenvalues. This is equivalent to the eigenvalue of $K_n$. The newly appearing eigenvalue is the solution of  
\begin{align}
0=(1 -\gamma^*)-\tilde{k}_{(k+1)}^\top(K_{(k)}-\gamma^*E_k)^{-1}\tilde{k}_{(k+1)}
\end{align}
This is also strictly positive and different from $(\gamma_1,\dots,\gamma_k)$.
We express the solution of the above by $\gamma_{k+1}$
Thus, we want to evaluate the following value
\begin{align}\label{variance}
&\mathrm{MMD}(\{(w_\mathrm{BQ}^{(n)},x_n)\}_{n=1}^{k+1})^2- \mathrm{MMD}(\{(w_\mathrm{BQ}^{(n)},x_n)\}_{n=1}^{k})^2\nonumber
\end{align}
To check this value, in addition to the eigenvalue, we also check the eigenvector of the gram matrix $K_{(k+1)}$.
From Eq.(\ref{gram_k}), we expand this matrix. For simplicity, we express $\bold{a}=\tilde{k}_{(k+1)}$
\begin{align}
K_{(k+1)}=Q^\top\left(
    \begin{array}{cc}
      K_{(k)} & 0  \\
      0 & 1-\bold{a}^\top K_{(k)}^{-1} \bold{a}
    \end{array}
  \right)Q,
\end{align}
where
\begin{align}
Q=\left(
    \begin{array}{cc}
      E_{k} & K_{(k)}^{-1} \bold{a}  \\
      0 & 1
    \end{array}
  \right)
\end{align}
Here we consider the following where $d=1-\bold{a}^\top K_{(k)}^{-1} \bold{a}$ for simplicity,
\begin{align}
\left(
    \begin{array}{cc}
      K_{(k)} & 0  \\
      0 & d
    \end{array}
\right)
\left(
    \begin{array}{c}
      e_{i} \\
      0 
    \end{array}
\right)=
\left(
    \begin{array}{c}
      K_{(k)}e_{i} \\
      0 
    \end{array}
\right)
=\lambda_i
\left(
    \begin{array}{c}
      e_{i} \\
      0 
    \end{array}
\right)
=\lambda_ie_i'\nonumber
\end{align}
Thus, this $e_i'$ can be regarded as the eigenvector whose eigenvalue is $\lambda_i$. Also, by noticing the fact that
\begin{align}
\left(
    \begin{array}{cc}
      K_{(k)} & 0  \\
      0 & d
    \end{array}
\right)
\left(
    \begin{array}{c}
      0 \\
      1 
    \end{array}
\right)=d
\left(
    \begin{array}{c}
      0 \\
      1 
    \end{array}
\right)
=de_{k+1}\nonumber
\end{align}
This is also the eigenvector whose eigenvalue is $d$, also this is orthogonal to $e_{1}',\dots,e_{k}'$.
By setting $U'=(e_1',\dots,e_k',e_{k+1})$, we can diagonalize $K_{(k+1)}$ as
\begin{align}
K_{(k+1)}&=Q^\top U'\left(
    \begin{array}{cccc}
      \gamma_1 & \dots && 0 \\
       & \ddots &&  \\
       &  & \gamma_k &\\
      0& \dots && d \\
    \end{array}
  \right)U'^\top Q\\
&=Q^\top U'\Gamma' U'^\top Q
\end{align}
Thus,
\begin{align}
K_{(k+1)}^{-1}=Q^\top U'\Gamma'^{-1} U'^\top Q
\end{align}
Let us calculate $U'^\top Q$ furthur,
\begin{align}
U'^\top Q=
\left(
    \begin{array}{cc}
      U^\top & 0  \\
      0 & 1
    \end{array}
\right)
\left(
    \begin{array}{cc}
      E_{k} & K_{(k)}^{-1}\bold{a}  \\
      0 & 1
    \end{array}
\right)
=\left(
    \begin{array}{cc}
      U^\top & U^\top K_{(k)}^{-1}\bold{a}  \\
      0 & 1
    \end{array}
\right)\nonumber
\end{align}
Let us multiply $(z_1,\dots,z_k,z_{k+1})^\top=(\bold{z},z_{k+1})^\top$ to the above expression,
\begin{align}
U'^\top Q\left(
    \begin{array}{c}
      \bold{z} \\
      z_{k+1}
    \end{array}
\right)
&=\left(
    \begin{array}{cc}
      U^\top & U^\top K_{(k)}^{-1}\bold{a}  \\
      0 & 1
    \end{array}
\right)\left(
    \begin{array}{c}
      \bold{z} \\
      z_{k+1}
    \end{array}
\right)\nonumber \\
&=\left(
    \begin{array}{c}
      U^\top\bold{z} \\
      0
    \end{array}
\right)+z_{k+1}
\left(
    \begin{array}{c}
      U^\top K_{(k)}^{-1}\bold{a} \\
      1
    \end{array}
\right) \nonumber
\end{align}
Based on these results, let us calculate how the variance changes when we add the one data in BQ. Let us go back to Eq.(\ref{quadratic}),
\begin{align}\label{variance}
&\mathrm{MMD}(\{(w_\mathrm{BQ}^{(n)},x_n)\}_{n=1}^{k+1})^2- \mathrm{MMD}(\{(w_\mathrm{BQ}^{(n)},x_n)\}_{n=1}^{k})^2\nonumber \\
&=-z_{k+1}^2 \left((U^\top K_{(k)}^{-1}\bold{a})^\top, 1\right)\Gamma^{-1'} \left(
    \begin{array}{c}
      U^\top K_{(k)}^{-1}\bold{a} \\
      1
    \end{array}
\right)\nonumber \\
&=-\alpha z_{k+1}^2\nonumber \\
&<0
\end{align}
where it is clear that $\alpha >0$. 
Based on these results, let us describe the proof of theorem 1 for the inexact step sizes. 
\begin{proof}
For the notation, let us denote $\mathrm{MMD}(\{(w_\mathrm{BQ}^{(n)},x_n)\}_{n=1}^{k+1})^2$ as $\|v_{k+1}\|^2$. Then, now we want to measure how the convergence rate is affected by the inexact step size, first, we will check the ratio of the variance between $k$-th and $k+1$-th step,
\begin{align}
\frac{\|v_{k+1}\|^2}{\|v_{k}\|^2}=1-\frac{\alpha z_{k+1}^2}{\mathbb{E}_{x\sim p(x)}k(x,x')-\sum_{i=1}^k\gamma_i^{-1}z_i'^\top z_i'}
\end{align}
Remember that this is the similar expression in the proof of the line searh FW. 
Since the convergence speed of BQ is at least faster than the line search, the convergence coefficient of BQ is larger than that of the line search, we can say that
\begin{align}
\frac{R^2}{(\|g\|+\rho_s\|M\|)^2}\leq\frac{\alpha z_{k+1}^2}{\mathbb{E}_{x\sim p(x)}k(x,x')-\sum_{i=1}^k\gamma_i^{-1}z_i'^\top z_i'}
\end{align}
Now let us consider the empirical approximation effect. We express it via the ratio $\beta_i=\sum_l k(x_i,x_l)/\int k(x_i,x')p(x')dx$, which is the ration between exact weight and empirical approximation. Then if we use the approximate BQ step, the approximated $\frac{\|\tilde{v}_{k+1}\|^2}{\|\tilde{v}_{k}\|^2}$ (we stress that those $\tilde{v}$s are the variance which is calculated based on the approximated BQ weights) can be written as
\begin{align}
&\frac{\|\tilde{v}_{k+1}\|^2}{\|\tilde{v}_{k}\|^2}\nonumber \\
&=\frac{\beta_{k+1}^2\alpha z_{k+1}^2}{\mathbb{E}_{x\sim p(x)}k(x,x')-\sum_{i,j=0}^n\beta_jz_jK^{-1}_{ij}\beta_iz_i} \nonumber \\
&=\frac{\alpha z_{k+1}^2}{\mathbb{E}_{x\sim p(x)}k(x,x')-\sum_{i=1}^k\gamma_i^{-1}z_i'^\top z_i'}\nonumber \\
&\times\frac{\beta_{k+1}^2(\mathbb{E}_{x\sim p(x)}k(x,x')-\sum_{i=1}^k\gamma_i^{-1}z_i'^\top z_i')}{\mathbb{E}_{x\sim p(x)}k(x,x')-\sum_{i,j=0}^n\beta_jz_jK^{-1}_{ij}\beta_iz_i}
\end{align}
To assure the geometric behavior, $\mathbb{E}_{x\sim p(x)}k(x,x')-\sum_{i,j=0}^n\beta_jz_jK^{-1}_{ij}\beta_iz_i$ must be positive. (Since $\mathbb{E}_{x\sim p(x)}k(x,x')-\sum_{i=1}^k\gamma_i^{-1}z_i'^\top z_i'$ is always positive.) If this condition is satisfied then we express
\begin{align}
\delta_{\mathrm{BQ}}=\frac{\beta_{k+1}^2(\mathbb{E}_{x\sim p(x)}k(x,x')-\sum_{i=1}^k\gamma_i^{-1}z_i'^\top z_i')}{\mathbb{E}_{x\sim p(x)}k(x,x')-\sum_{i,j=0}^n\beta_jz_jK^{-1}_{ij}\beta_iz_i}
\end{align}
which is some positive constant. Then
\begin{align}
\frac{\delta_{\mathrm{BQ}}R^2}{(\|g\|+\rho_s\|M\|)^2}\leq\frac{\delta_{\mathrm{BQ}}\alpha z_{k+1}^2}{\mathbb{E}_{x\sim p(x)}k(x,x')-\sum_{i=1}^k\gamma_i^{-1}z_i'^\top z_i'}
\end{align}
holds. This ends the proof of theorem 1 and theorem2 of the case of finite dimension.

Next we consider the infinite dimensional RKHS. By using the above notation, when we use the inexact BQ step size, 
\begin{align}
\hat{\Delta}_{BQ}:=\|v_{k+1}\|^2-\|v_{k}\|^2=-\beta_{k+1}^2\alpha z_{k+1}^2
\end{align}
In the same way, we express the above quantity under exactly calculated BQ step as
\begin{align}
\Delta_{BQ}:=\|v_{k+1}\|^2-\|v_{k}\|^2=-\alpha z_{k+1}^2
\end{align}
This quantity is the minimum thus from the line search results, we can say
\begin{align}
\Delta_{BQ}\leq\min_\gamma\left\{-\gamma\delta\|v_k\|^2 +\frac{\gamma^2}{2}(2\|M\|\rho_s)^2\right\} 
\end{align}
Then, following relation holds,
\begin{align}
\hat{\Delta}_{BQ}\leq\beta_{k+1}^2\min_\gamma\left\{-\gamma\delta\|v_k\|^2 +\frac{\gamma^2}{2}(2\|M\|\rho_s)^2\right\} 
\end{align}
To eliminate the dependence of $k$ form $\beta_{k+1}$, let $\beta'$ is the largest of $(\beta_1, \dots, \beta_{k+1})$. And in the same way as the previous discussion, we can conclude by the induction that
\begin{align}
\|v_k\|^2\leq2\frac{(1+\beta'^2\delta)(2\|M\|\rho_s)^2}{\delta(\beta'^2\delta k+2)}
\end{align}
If we set $\delta_{\mathrm{BQ}}=\beta'^2$, this ends the proof.
\end{proof}

\section{Proof of Theorem~3 in the main paper}\label{sup:cont_proof}
Our results are directly obtained by AppendixB of \cite{FWBQ}. We use the proof of the contraction theorem.
The calculations after Eq.(26) in \cite{FWBQ} are held in our case and Eq.(31) in \cite{FWBQ} holds in our situation. Thus all we need to do is to substitute the variance of ours into Eq.(31) in \cite{FWBQ}. Our variance is derived by reweighting the particles obtained by MMD-FW with Bayesian Quadrature weight and calculate the weighted MMD. This is upper bounded by the bound of the result Theorem~2 in the main paper because MMD which is calculated by Bayesian quadrature weight is optimal. Thus, by substituting the result of Theorem~2 in the main paper into Eq.(31) in \cite{FWBQ}, we obtain the result.

Actually, we cannot calculate the Bayesian Quadrature weight analytically, so we approximate it by obtained particles. Even in such a case, we can obtain the upper bound.
The posterior distribution is denoted by $N(Z_{f,\hat{p}},\sigma_N^2)$, where $\sigma_N=\mathrm{MMD}(\{x_i,w_i^{\mathrm{BQ}}\}_{i=1}^N)$, and $w_i^{\mathrm{BQ}}$ is the Bayesian Quadrature weight. Since we approximate this weight empirically and denote the corresponding variance by $\hat{\sigma}_N=\mathrm{MMD}(\{x_i,\hat{w}_i^{\mathrm{BQ}}\}_{i=1}^N)$. Since Bayesian Quadrature weight is the optimal weight,
$\sigma_N\leq\hat{\sigma}_N$. Thus we can upper bound Eq.(31) in \cite{FWBQ} by this variance whose weight is approximated by particles.

\section{Kernel selection}\label{sup:kernel}
The choice of the kernel is crucial numerically and theoretically.
In the above convergence proof, 
we assumed that within the affine hull $\mathcal{M}$,
there exists a ball with center $\hat{x}$ and radius R that is included in $\mathcal{M}$.
\cite{bach_herding_equi, FWBQ} proved that for infinite dimensional RKHS, 
such as the case of RBF kernel, 
this assumption never holds. 
Thus, we can only have the sub-linear convergence for RBF kernels in general.
However, as pointed in \cite{FWBQ}, 
even if we use RBF kernels, 
thanks to the rounding in a computer, 
what we treat in a simulation are actually finite dimensional. 
This holds to our situation, 
and in the experiments, we observed the linear convergence of our algorithm.

\section{Details of Experiments}\label{sup:exps}
In this section, we show the detail experimental settings and results which we cannot show in the main paper due to the limit of the space.

\subsection{Algorithm of SVGD}\label{sup:svgd}
We implemented SVGD by following pseudo codes.

\hspace{3.0mm}
\begin{algorithm}[ht]\small
   \caption{Stein Variational Gradient Descent}
   \label{alg:stein_VI}
\begin{algorithmic}[1]
   \STATE {\bfseries Input:} A target density $p(x)$ and  initial particles $\{x_n^0\}_{n=1}^N$
   \STATE {\bfseries Output:} Particles $\{x_i\}_{i=1}^n$ which approximate $p(x)$
   \FOR{iteration $l$}
   \STATE $x_{n}^{(l+1)}\leftarrow x_n^{(l)}+\epsilon^{(l)}\hat{\phi}^*(x_n^{(l)})$, where $\hat{\phi}^*(x)=\frac{1}{N}\sum_{n=1}^N\left[k(x_n^{(l)},x)\nabla_{x_n^{(l)}}\ln p(x_n^{(l)})+\nabla_{x_n^{(l)}}k(x_n^{(l)},x)\right]$
   \ENDFOR
\end{algorithmic}
\end{algorithm}

\subsection{Toy dataset}\label{sup:toydata}
\begin{figure}[t]
 \vspace{-3.5mm}
 \centering
 \hspace{-0.in}
 \subfigure[2D gaussian with particles obtained by MMD-FW $h=0.2$]{
 \includegraphics[width=0.8\linewidth]{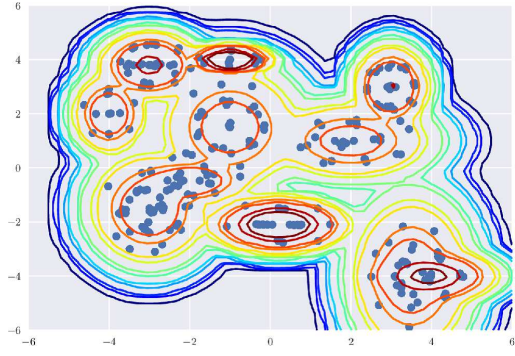}
 \label{fig:toy_pic}}
 \subfigure[The dependence of number of particles]{
 \includegraphics[width=0.8\linewidth]{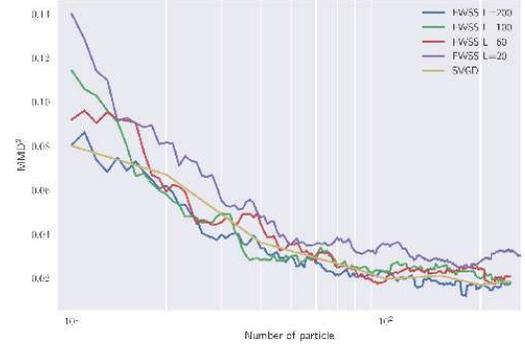}
 \label{MMD_plot}}
 \caption{The results of the toy data by MMD-FW of fixed bandwidth}
 \vspace{-0.2in}
\label{2Dtoy}
\end{figure}

To clarify how our method works, 
we checked our algorithm with a two dimensional toy dataset. 
The true distribution is a two dimensional mixture of Gaussians which is composed of 11 Gaussian distributions. 
Here we applied our MMD-FW and observed how the particles are fitted to the distribution. 
The results are shown in Fig~\ref{2Dtoy}. 
In this figure, the number of particle is 200 and $L=50$ and $h=0.3$. 
We also changed $L$, the number of gradient descent in the approx-LMO, 
and compared how MMD decreases with SVGD in Fig~\ref{MMD_plot}. 
We found that both our method and SVGD decreases linearly.

The selection of the bandwidth is crucial for the success of the method. 
As we had shown in the main paper, there are three ways to specify the bandwidth.
The first choice is using the fixed bandwidth and this choice is often used in Bayesian quadrature, e.g., \cite{FWBQ}. 
The second choice is the median trick which is used in SVGD \cite{liu2016stein}. 
This method enables us to choose the bandwidth adaptively during the optimization. 
The third choice is using the gradient descent for $h$ to minimize the kernelized Stein discrepancy during optimization. This is used in \cite{jitkrittum2017linear}.

In Fig~\ref{2Dtoy_other_band}, we showed other selection of bandwidth. 
As shown in Fig~\ref{fig:toy_pic_01}, 
the small bandwidth makes the particles sparsely scattered.
This is due to the fact that the second term of the update equation, 
which corresponds to the regularization term becomes very large.
Thus particles tend to take a distance from each other. 
If we take a large  bandwidth, 
the regularization term becomes small, 
and thus particles become close to each other which is shown in Fig~\ref{fig:toy_pic_02}.

\begin{figure}[t]
 \vspace{-3.5mm}
 \centering
 \hspace{-0.in}
 \subfigure[2D gaussian with particles obtained by MMD-FW bandwidth $h=0.1$]{
 \includegraphics[width=0.8\linewidth]{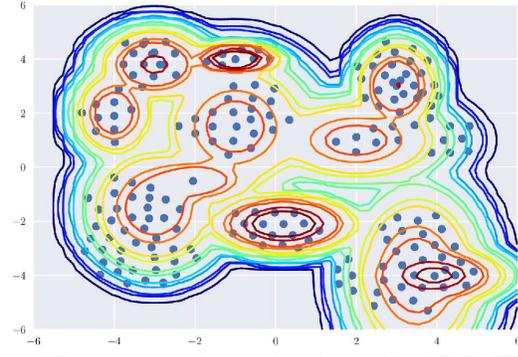}
 \label{fig:toy_pic_01}}\\
 \subfigure[2D gaussian with particles obtained by MMD-FW bandwidth $h=1.0$]{
 \includegraphics[width=0.8\linewidth]{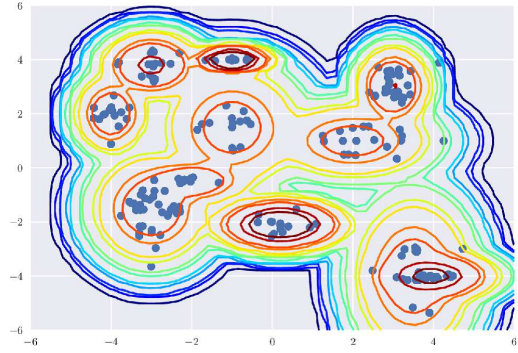}
 \label{fig:toy_pic_02}}
 \caption{The results of the toy data by MMD-FW of fixed bandwidth}
 \vspace{-0.2in}
\label{2Dtoy_other_band}
\end{figure}

Finally, we showed the results of MMD-FW by median trick and SVGD, the result is shown in Fig.~\ref{2Dtoy_median}

\begin{figure}[t]
 \vspace{-3.5mm}
 \centering
 \hspace{-0.in}
 \subfigure[2D gaussian with particles obtained by MMD-FW bandwidth by median trick]{
 \includegraphics[width=0.8\linewidth]{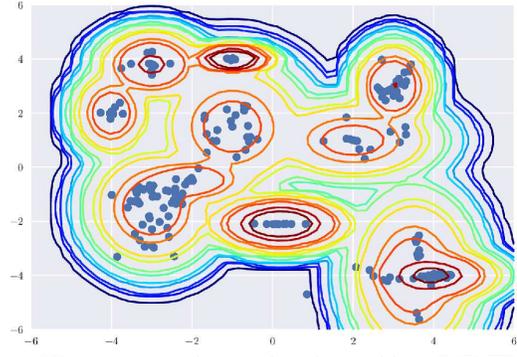}
 \label{fig:toy_pic_05}}\\
 \subfigure[2D gaussian with particles obtained by SVGD]{
 \includegraphics[width=0.8\linewidth]{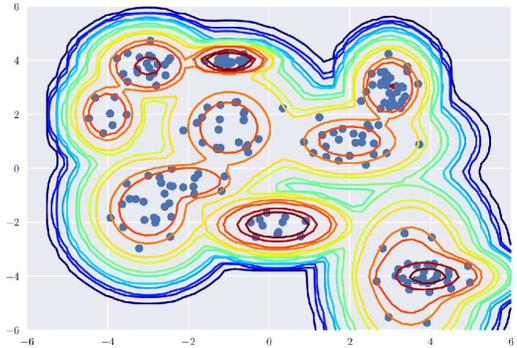}
 \label{fig:toy_pic_06}}
 \caption{Boundaries of logistic regression using ordinary VI and the proposed method}
 \vspace{-0.2in}
\label{2Dtoy_median}
\end{figure}

\subsection{Details of the Benchmark experiments}
\subsubsection*{Bayesian logistic regression}
In this experiment, we used Adam with a learning rate of 0.005 and we split the data, 90$\%$ are used for training and 10$\%$ are used for the test. Minibatch size is 100. For the LMO calculation, we set $L=250$. We used the median trick for the kernel bandwidth. To calculate the MMD, we have to fix the bandwidth of the kernel. We simply take the median of the bandwidth which changes adaptively during the optimization. In the main paper, we used 2.5. To calculate the MMD, we generate samples by Hybrid Monte Carlo (HMC).

\subsubsection*{Bayesian neural net regression}
In this experiment, we used Adam with a learning rate of 0.005 and we split the data, 90$\%$ are used for training and 10$\%$ are used for the test. minibatch size is 100 except for year dataset, where we used 500 minibatch sizes.

We use the zero mean Gaussian for the prior of the weights and we put $\mathrm{Gamma}(1, 0.1)$ prior for the inverse covariances.

For the LMO calculation, we set $L=1000$ except for year dataset where we set $L=2000$. Since it is difficult to obtain the MAP estimation for the initial sates, this corresponds to the case of Non-MAP initialization as we explained the previous section.

In the main text, we showed the figure of naval dataset and here, we show that of the protein data in Fig.~\ref{fig:protein_result}.
\begin{figure}[t]
 \centering
 \hspace{-0.in}
 \subfigure{
 \includegraphics[width=1.0\linewidth]{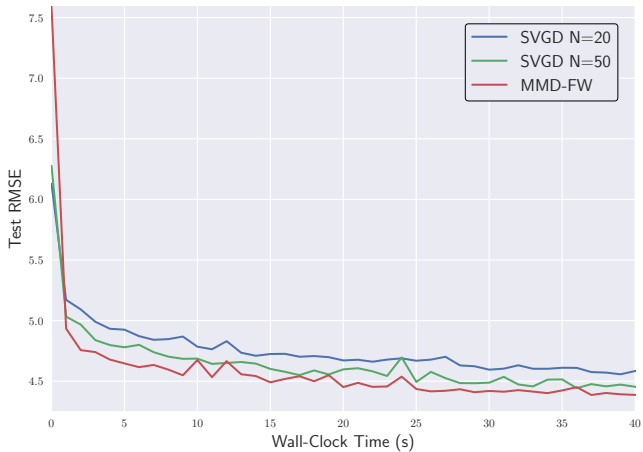}
 }
\caption{Comparison of MMD-FW and SVGD in terms of wall clock time with test accuracy(Protein data)}
\label{fig:protein_result}
\end{figure}

In the main paper, we checked the performance of MMD-FW and SVGD by fixing the computation time. Those fixed time are selected that the test RMSE is sufficiently converged.

\subsection{MMD comparison between MMD-FW and SVGD}\label{sup:experiment_mmd}
We changed the number of particles and plot it. The value of MMD at each number of particles are calculated after the 30000 steps where the optimization had finished. The result is shown in Fig.\ref{MMD_plot}.

To calculate the MMD, we also generate ``true samples'' by HMC. SVGD have smaller MMD compared to ours. This is due to the fact that SVGD simultaneously optimizes all particles and try to put particles in the best position which corresponds to the global optima. On the other hand, MMD-FW just increases the particles greedily and thus this results in a local optimum. Hence the better performance of SVGD compared to MMD-FW at the same number of particles in terms of MMD is a natural result.

\section{Stein Points Experiments}
The stein pints method utilizes two algorithms for the objective function,
the greedy algorithm and the herding algorithm.
In order to perform optimization,
the authors indicate three methods:
the Nelder-Mead method, the Monte Carlo method and the grid search method.
To perform on step optimization in a $n$ dimension parameter space,
the Nelder-Mead method needs to evaluate the objective function at least $n+1$ times,
and the grid search method needs to evaluate $d^n$ times,
where $d$ is the number of the grids in one dimension.

We conducted toy experiments to approximate Gaussian mixtures in $6$ different combinations.

\begin{figure}[tb]
 \centering
 \hspace{-0.in}
 \subfigure[Greedy Monte Carlo]{
 \includegraphics[width=0.45\linewidth]{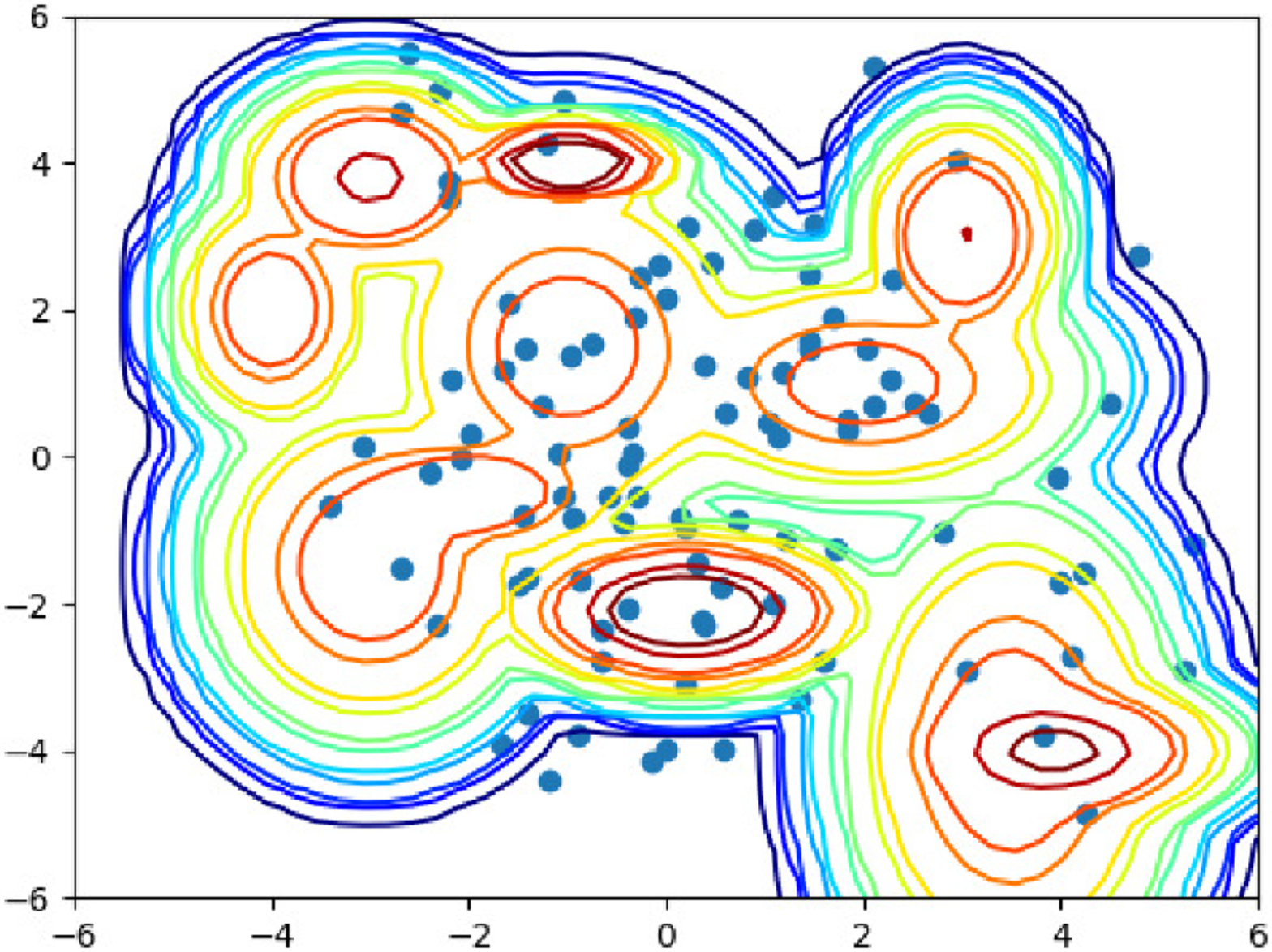}}~
 \subfigure[Herding Monte Carlo]{
 \includegraphics[width=0.45\linewidth]{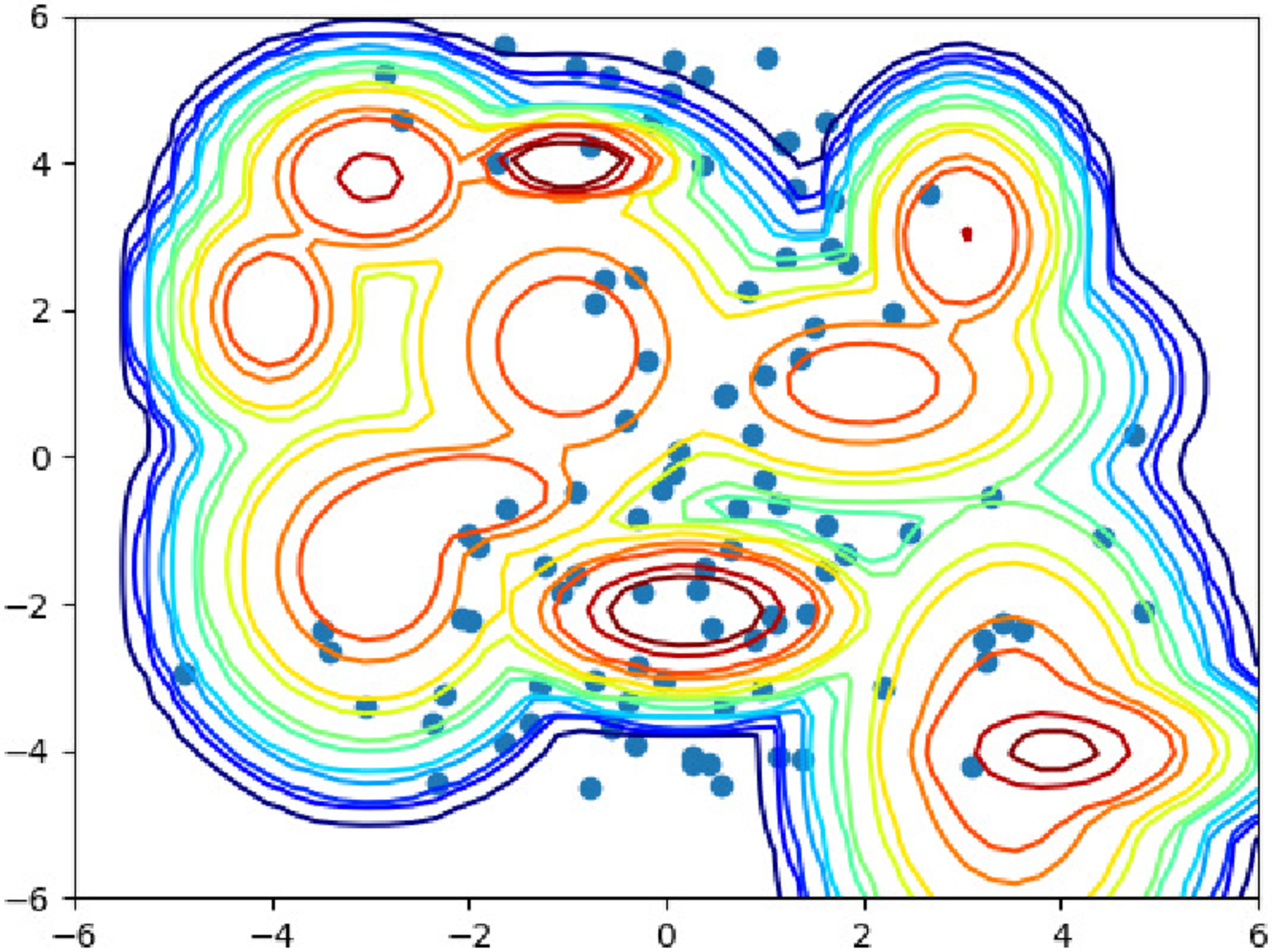}}\\
  \subfigure[Greedy Nelder-Mead]{
 \includegraphics[width=0.45\linewidth]{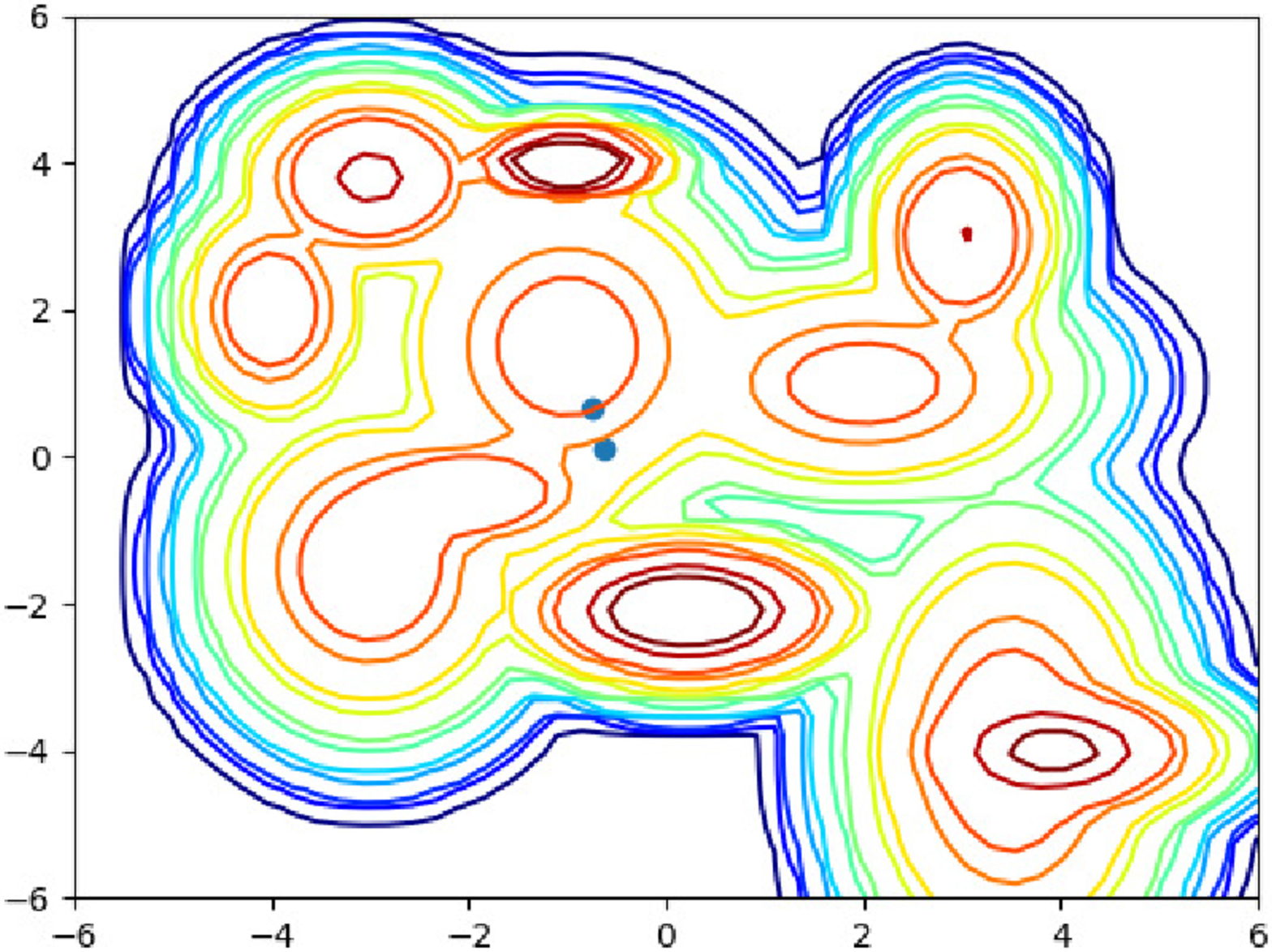}}~
 \subfigure[Herding Nelder-Mead]{
 \includegraphics[width=0.45\linewidth]{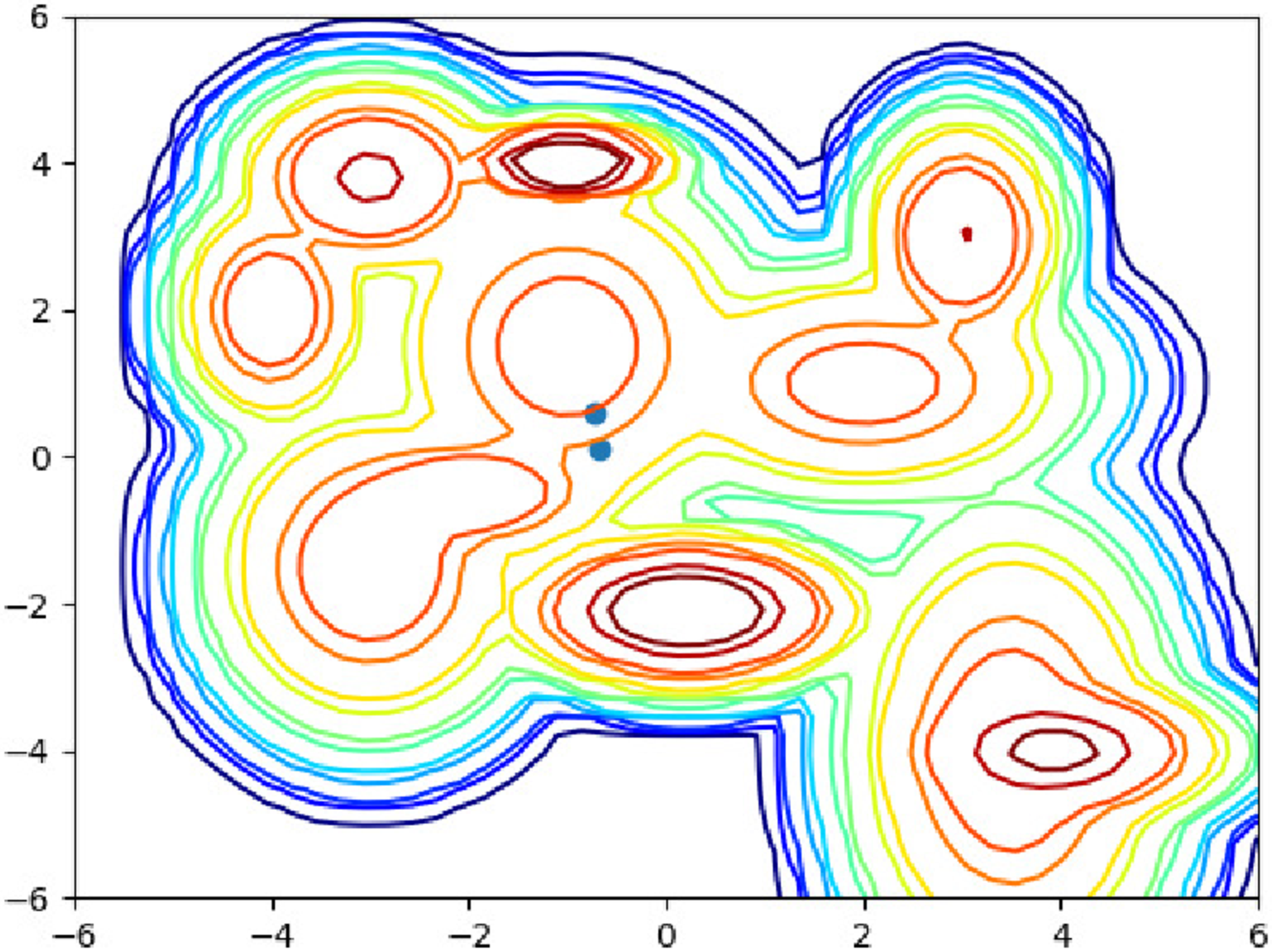}}\\
  \subfigure[Greedy Grid Search]{
 \includegraphics[width=0.45\linewidth]{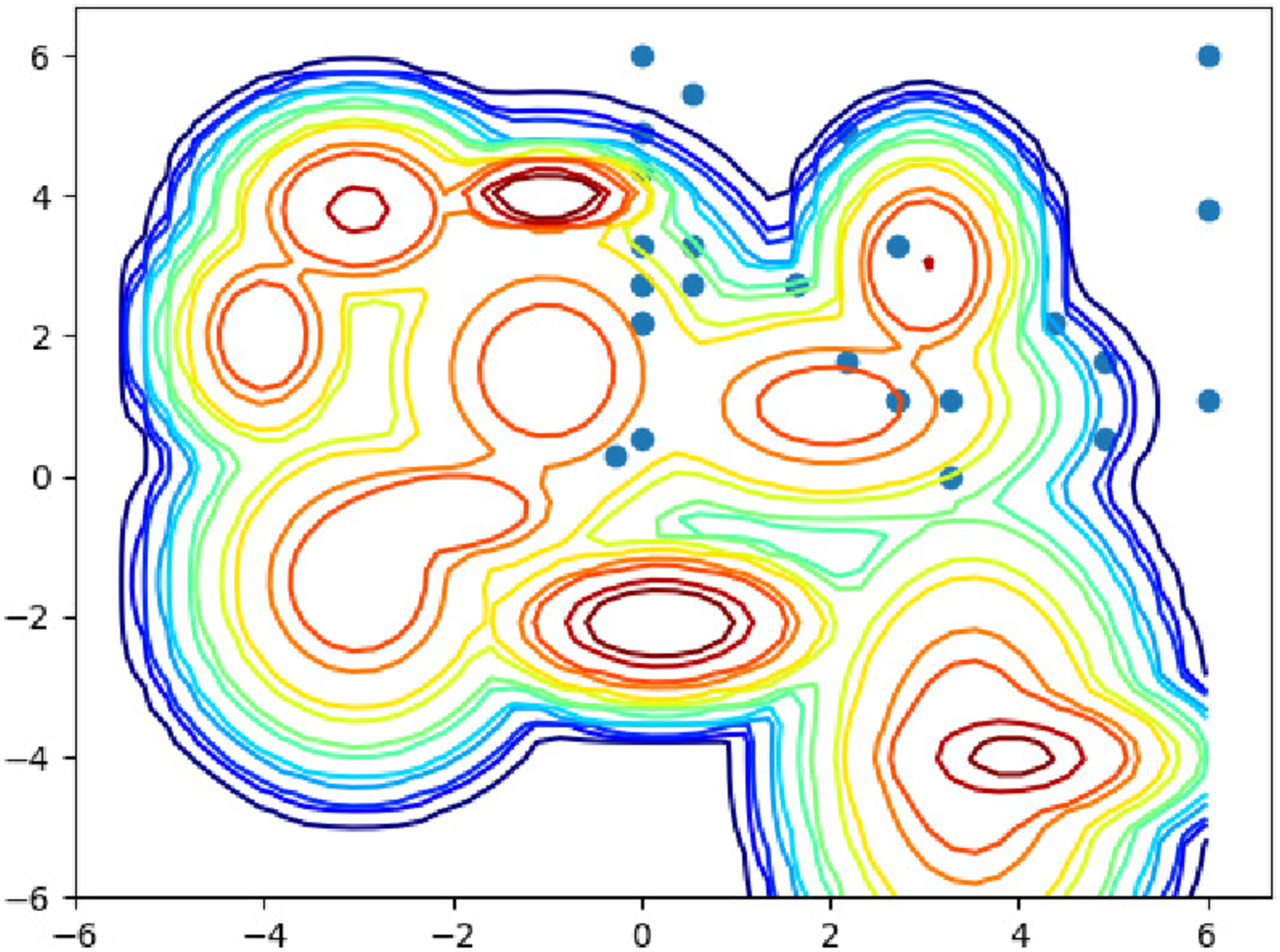}}~
 \subfigure[Herding Grid Search]{
 \includegraphics[width=0.45\linewidth]{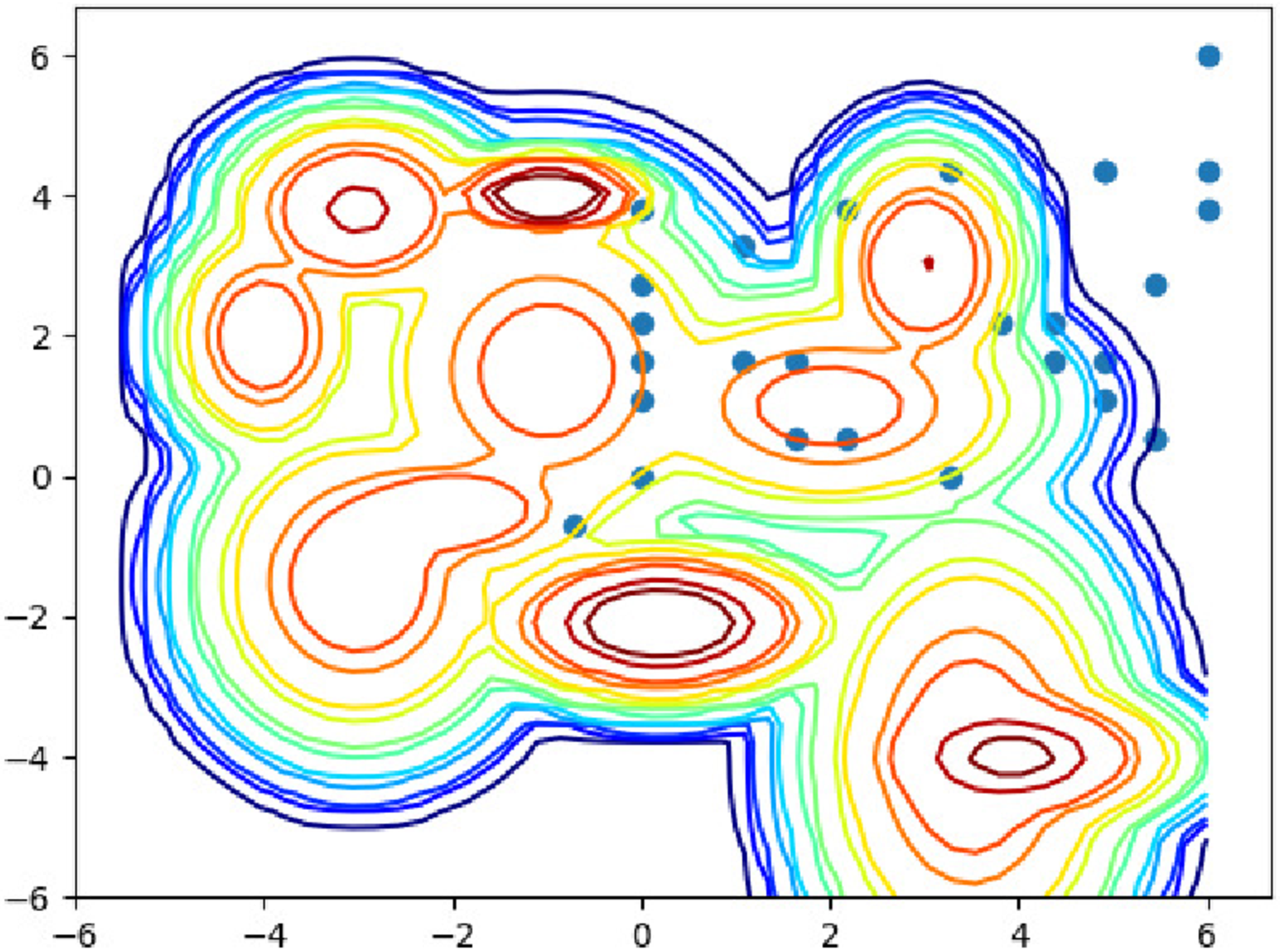}}\\
 \caption{Plots of the toy experiment}
 \label{fig:toy-sp}
\end{figure}

The result shown in Fig. \ref{fig:toy-sp} is not favorable as expected.
First,
the data points need to be bounded in a specifically designed area,
which is not trivial and may require more time than usual.
Second,
it failed to capture the character of the target distribution since it only does exploration.
Since the greedy algorithm together with Monte Carlo seems to perform the best fit,
we use this setting in the Bayesian logistic regression experiment.

The details of the greedy algorithm are elaborated as follows.
The distance we want to minimize between $n$ sampled points and the target distribution is defined as
$$\textrm{D} = \sqrt{\frac1{n^2} \sum_{i,j=1}^n k_0(x_i, x_j)}.$$
Starting from the first MAP data point,
the greedy algorithm tries to find a data points that minimize the distance.
The greedy algorithm solves the optimization problem
$$x_n = \textrm{argmin}_{x} \frac{k_0(x, x)}2 + \sum_{i=1}^{n-1} k_0(x_i, x)$$
for the $n$th data point,
where $k_0(x, x')$ is the Stein repoducing kernel defined as
\begin{align}
k_0(x, x') & = \nabla_x\cdot\nabla_{x'}k(x,x') + \nabla_x k(x, x') \cdot \nabla_{x'} \log p(x') +\nonumber \\
 &\nabla_{x'} k(x, x') \cdot \nabla_x \log p(x) + \nonumber \\
 &k(x, x') \nabla_x \log p(x) \nabla_{x'} \log p(x').
\end{align}
The Gaussian kernel is used for the base kernel $k(x, x')$.

The details of the Monte Carlo methods are elaborated as follows.
The first sample is drawn by performing MAP approximation,
for which we looped $100$ times.
From the second sample,
we take the strategy below.
First,
we uniformly select $20$ base points within existing points.
Then,
we sample $20$ points from a Gaussian distribution,
whose location is the base point and scale is set to be $1$.
We resampled the points until the elements of the $20$ points all fall in the range $[-1, 1]$.
Finally,
we evaluate the $20$ points and select the one performs the best.

However, the experiment is hardly feasible.
Sampling only $4$ data points took $3$ minutes and the accuracy is only $56$\%.

We also tried to test the method on Bayesian neural network settings.
However, it is not realistic since the dimension of the parameter space is too large.

\section{PBC-MMD-FW}\label{sup:APBCMMD-FW}
In FWSS, we can combine a more practically useful variant of FW with our problem setting. When we consider high-dimensional problems, SP does not work.
The computational costs per iteration of SVGD and FWSS are $\mathcal{O}(D^2N^2)$ and $\mathcal{O}(D^2N)$ when the computational costs of $\nabla\ln p(x)$ is $\mathcal{O}(D^2)$, where $D$ is the dimension of the target density. For a high-dimensional problem, we need many particles to approximate the high-dimensional target density $p(x)$. For these reasons, SVGD and FWSS are difficult to apply.
To overcome this limitation due to the high-dimensionality, we use a variant of the FW algorithm, the asynchronous parallel block-coordinate FW (AP-BCFW) algorithm \cite{wang2016parallel}. An advantage of AP-BCFW is that we can optimize the block separable domain problem in parallel. The problem, $\min_x f(x)$ s.t. $x=[x_1,\dots,x_d]\in \times^d_{i=1}\mathcal{M}_i$ can be broken into $d$ independent sub-LMO problems, $\min _{s_i\in \mathcal{M}_i}\langle s_i,\nabla_i f(x) \rangle$. 

This proposed algorithm converges sublinearly or linearly depending on the step width, which is discussed in Appendix.
The computational cost per iteration per worker in APBC-FWSS is $O(d^2N)$, where $d$ is the number of dimensions assigned to each worker, which is $d\leq D$.

In our problem setting, when we use specific kernels, we can factorize the kernel as $k(x,y)=\prod_{i=1}^D k(x_{(i)},y_{(i)})$.
This means that our problem is $J(\mu_{\hat{p}})=\frac{1}{2}\|\mu_p-\mu_{\hat{p}}\|_\mathcal{H}^2$, $\mathcal{H}=\times^d_{i=1}\mathcal{H}_i$. 
Based on this factorization, we propose parallel block coordinate FWSS (PBC-FWSS) given in Alg.~\ref{alg:APBCFWSS}. In the algorithm, each worker calculates $\bar{g}_{n(i)}=k(x_{(i)},\cdot)$, which corresponds to deriving the $i$-th dimension of a particle.

\begin{algorithm}\small
   \caption{Parallel Block Coordinate FWSS (PBC-FWSS)}
   \label{alg:APBCFWSS}
\begin{algorithmic}[1]
   \STATE {\bfseries Input:} A target density $p(x)$, a set of workers $\mathcal{O}$.
   \STATE Calculate MAP estimation for $\mu_{\hat{p}}^{(0)}$
   \STATE {Broadcast $\mu_{\hat{p}}^{(0)}$ to all workers in $\mathcal{O}$}
   \FOR{iteration $n$}
   \STATE {Each oracle calculates $\bar{g}_{n(i)}=\mathrm{Approx-LMO}(\mu_{\hat{p}}^{(n)})\left.\right|_{i=1}^D$}
   \STATE {Keep reserving $(i,\bar{g}_{n(i)})$ from $\mathrm{O}$}
   \STATE {Update $\mu_{\hat{p}}^{(n+1)}=$Fully Correction($\otimes_i^D\bar{g}_{n(i)},\mu_{\hat{p}}^{(n)}$)}
   \STATE Broadcast $\mu_{\hat{p}}^{(n+1)}$ to $\mathcal{O}$.
   \ENDFOR
\end{algorithmic}
\end{algorithm}

In our situation, the domain is the marginal polytope $\mathcal{M}$ of the RKHS  $\mathcal{H}$, which is defined as the closure of the convex hull of $k(\cdot,x)$. Thus we cannot directly apply the block coordinate algorithm to our setting. 
This can be easily confirmed in the following way.
First we denote the convex hull in RKHS as $\sum_i\alpha_ik(x_i,\cdot)=\sum_i\alpha_i\phi_i$ where $\alpha_i$ is a positive coefficient and this is expressed as
\begin{align}
\sum_i\alpha_i\phi_i=\sum_i\alpha_i\otimes_d\phi_i^d,
\end{align}
where we assume some special kernels that we can factorize it into the Cartesian product such as RBF kernels which satisfies $k(x,y)=\prod_{i=1}^d k(x_{(i)},y_{(i)})$.

And to apply the block coordinate algorithm, each domain is the convex hull, that is,
\begin{align}
\otimes_d\left(\sum_{i^d}\alpha_{i^d}\phi_{i^d}\right)=\sum_{i}\sum_{j}\sum_{k}\ldots \alpha_{i}\alpha_{j}\alpha_{k}\ldots (\phi_i\otimes\phi_j\otimes\phi_k\ldots).
\end{align}
Thus, this is a much larger space than our problem. To mediate this situation, we have to consider the synchronize algorithm instead of the asynchronous one.
That is if we can separate our objective into each dimension by using the Cartesian product, we will first solve the LMO in each dimension separately, and wait all the worker finish each calculation. After that, we combine all the results of the workers and update the $\mu_{\hat{p}}^{(n)}$. By using this strategy, in each LMO, we solve subproblems in the convex hull, and can easily find the solution at the extreme points in each dimension, and the synchronous update makes the particles in $\sum_i\alpha_i\otimes_d\phi_i^d$. The drawback of this strategy is that if some workers are slow, we have to wait for them.

The next concern is whether we can separate our objective into each dimension by using the Cartesian product. 
In general, this is impossible since MMD entails the $\mu_p$ and
\begin{align}
\mu_p=\mathbb{E}_{p(x)}[\otimes_d\phi(x_d)]=\int\otimes_d\phi(x_d)p(x_1,\ldots,x_D)dx_1\ldots dx_D.
\end{align}
and this is not separable in general. 
However, we can separate them under conditional probability conditions which is a special case, when we focus on the update equation which we use to solve the LMO, 
and especially the derivative with respect to the $d$-th dimension of $x$,
\begin{align}
&\nabla_{x^{(d)}}\mathrm{MMD}(x)^2\nonumber \\
&=\nabla_{x^{(d)}}\langle \mu_{\hat{p}}^{(n)}-\mu_p,g \rangle \nonumber \\
&=\frac{1}{n}\sum_{l=1}^{n}\nabla_{x^{(d)}}k(x_l,x)+\frac{1}{n}\sum_{l=1}^{n} k(x,x_l)\nabla_{x_l^{(d)}}\ln p(x_l) \nonumber \\
&=\frac{1}{n}\sum_{l=1}^{n}\prod_{i\neq d}k(x_l^{(i)},x^{(i)}) \nonumber\\
&\quad\quad\quad\left(\nabla_{x^{(d)}}k(x_l^{(d)},x^{(d)})+k(x_l^{(d)},x^{(d)})\nabla_{x_l}\ln p(x_l)\right).
\end{align}
So, for the $d$-th dimension, the effect of other dimensions comes from the coefficient $\prod_{i\neq d}k(x_l^{(i)},x^{(i)})$. 
When we use an RBF kernel and optimize the particle, 
the change of the  $\prod_{i\neq d}k(x_l^{(i)},x^{(i)})$ is much smaller than the change inside the bracket. 
Thus we fix the value of $\prod_{i\neq d}k(x_l^{(i)},x^{(i)})$ during the optimization. 
This enables us to separate the update of the particle in each dimension. 
This is especially useful when we use the median trick during the optimization, 
since the median trick tries to $\sum_jk(x_i,x_j)\sim N\exp{(-\mathrm{med}^2/h^2)}=1$ that is the sum of the coefficients of the gradients are coordinated to be 1. 
Hence the coefficient does not change so much during the iteration of the LMO.
Based on this assumption, we can separately solve the LMO in each dimension, and the obtained particle satisfies the accuracy parameter in approximate LMO,
\begin{align}
\langle \mu_{\hat{p}}^{(n)}-\mu_p,\tilde{g} \rangle=\delta \mathrm{min}_{g \in \mathcal{M}}\langle \mu_{\hat{p}}^{(n)}-\mu_p,g \rangle.
\end{align}
We can guarantee the algorithm in the same way as FWSS. We tested the algorithm on the benchmark dataset, and it seems work well. However, the success of this algorithm depends on the above assumption and the validity of the assumption depends on the kernel and its hyperparameters. The RBF kernel and median trick seem to satisfy the assumption and this is the reason the experiment worked well.

\subsection{Results of Other real datasets by PBC-MMD-FW}

We did the comparison of PBC-MMD-FW, MMD-FW, and SVGD for Bayesian neural net regression, where the model is one hidden layer with 100 units and Relu activation function. The data is Year dataset in UCI whose dataset size is 515344 and 91 features. The posterior dimension is 9203. We used minibatch size of 500 and optimize by Adam with 0.005 initial learning rate. For PBC-MMD-FW, we used 2 workers.

The results are shown in Fig~\ref{year_para}. Those two figures are the same plot, where the left one is the enlargement of the right figure at an early time.

As you can see that at an early time, PBC-MMD-FW is the fastest. However, as the optimization proceeds, the advantage of parallel computation had been disappeared. This might be due to our implementation in tensorflow that, for the parallel computation, we first separate the dimensions into each worker, this corresponds to the allocation of variables in tensorflow. Since this allocation makes the computation graph inefficient, and thus we did not gain so much speed up by PBC-MMD-FW.

\begin{figure}[t]
 \centering
 \hspace{-0.in}
 \subfigure[Comparison in terms of wall clock time with test RMSE]{
 \includegraphics[width=0.8\linewidth]{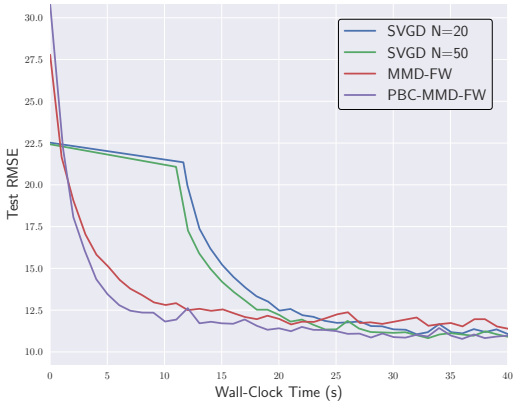}
 \label{fig:cov_time}}\\
 \subfigure[Comparison in terms of wall clock time with test RMSE]{
 \includegraphics[width=0.8\linewidth]{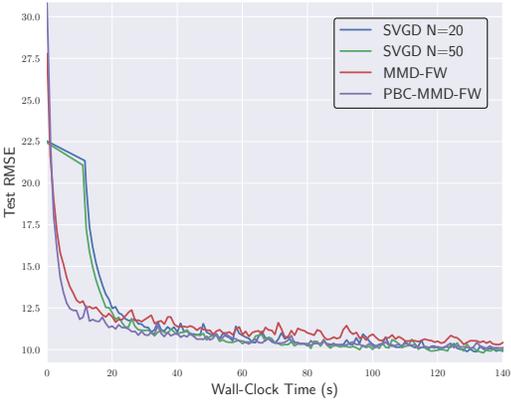}
 \label{fig:cov_mmd}}
 \caption{Comparison in the logistic regression model}
\label{year_para}
\end{figure}

\section{Cache-MMD-FW}\label{sup:CacheMMD-FW}
As we had discussed in the main paper of Section 4.1, we can combine MMD-FW and SVGD. The algorithm is simple. We just replace the Approx-LMO in MMD-FW by Cached approx-LMO as described in Alg.~\ref{alg:cache_LMO}. To use the Cached approx-LMO, we first optimize $N$ particles by SVGD. After finishing the SVGD, we store the optimized particles in the ``Cache''. Then, in the Cached approx-LMO, in each iteration, we first choose the particle which minimizes the absolute value of $\nabla_{x}\mathrm{MMD}(x)^2$ from the Cache. Then we adopt the chosen particle as the initial state of the solution and update it. By doing this, the number of iteration will be drastically small for each iteration. And we eliminate the chosen particle from the cache to prevent from choosing the same particle many times. Based on this Cached approx-LMO, the whole algorithm is given in Alg.~\ref{alg:cache-MMD-FW}. We name this algorithm, Cache-MMD-FW. When we use all the particles which are obtained by SVGD, then we will use the usual Approx-LMO in the Algorithm. The theoretical property of this algorithm is apparently as same as the MMD-FW.

\begin{algorithm}[ht]\small
   \caption{Cached approx-LMO}
   \label{alg:cache_LMO}
\begin{algorithmic}[1]
   \STATE {\bfseries Input:} $\mu_{\hat{p}}^{(k)}$
   \STATE {\bfseries Output:} $k(\cdot,x^{L+1})$
   \STATE $x^{(0)}=\mathrm{argmin}_{x\in\mathrm{cache}}|\nabla_{x}\mathrm{MMD}(x)^2|$
   \STATE Eliminate the chosen $x$ from the Cache
   \FOR{$l=0 \dots L$}
   \STATE Compute $\nabla_{x}\mathrm{MMD}^2$ by Eq.(5) in the main paper
   \STATE Update $x^{(l+1)}\leftarrow x^{(l)}+\epsilon^{(l)}\cdot\nabla_{x}\mathrm{MMD}^2$
   \ENDFOR
\end{algorithmic}
\end{algorithm}

\begin{algorithm}[ht]\small
 \caption{Cached MMD minimization by Frank-Wolfe algorithm}
 \label{alg:cache-MMD-FW}
\begin{algorithmic}
 \STATE $\dots$ as Alg.~3 in the main paper, except for the input of step 1 \\and use the Cached approx-LMO at step 3.
 \STATE {\bfseries Input:} A target density $p(x)$ and particles $\{x_n^{(0)}\}_{n=1}^n$ \\obtained by SVGD
 \STATE $\bar{g}_n= $Cached approx-LMO($\mu_{\hat{p}}^{(n)}$)
\end{algorithmic}
\end{algorithm}

We did the numerical experiment about this algorithm on the toy data which we had explained in the previous section. First, we optimized 200 particles by SVGD. We set the number of iteration $L=10$ in Cached approx-LMO. The results are shown in Fig~\ref{2dcache}.

\begin{figure}[t]
 \centering
 \hspace{-0.in}{
 \includegraphics[width=0.8\linewidth]{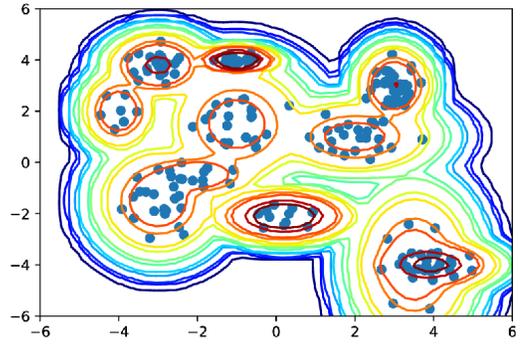}}
 \caption{Toy data results by Cached-MMD-FW}
\label{2dcache}
\end{figure}

\section{Lazy Frank-Wolfe algorithm}\label{sup:lazyFW}
As we had mentioned in the main paper, we can utilize the many variants of the FW to our setting. Here we pick up the Lazy FW \cite{lazy_FW}.

In Lazy FW, instead of calling the LMO at each step, we re-use the particles which had already been processed and stored in some of them satisfy the specified criterion. We call such a procedure as Lazy-LMO and shown in Alg \ref{alg:lazy_oracle}. Actually, this method never improves the sample complexity of the bound, however, it drastically reduces the wall-clock time. When no stored particles satisfy the criterion, we will solve the LMO or update the criterion. When we solve the LMO, we use the Chached approx-LMO of Alg.~\ref{alg:cache_LMO} which also contribute to the drastic reduction of the wall clock time of approximate LMO calculation.

To skip the calling of the LMO, we have to calculate the criterion. To calculate the criterion, we need to calculate the following expression, which is often called the duality gap,
\begin{align}\label{duality_gap}
\mathrm{Dg}(\mu_{\hat{p}}^{(n)},x):&=\langle \mu_{\hat{p}}^{(n)}-\mu_p,\mu_{\hat{p}}^{(n)}-\Phi(x) \rangle \nonumber \\
&=\sum_{l',l=0}^{n-1}w_{l'}^{(n-1)}w_l^{(n-1)}k(x_{l'},x_l)- \nonumber \\
&\quad \sum_{l=0}^{n-1}w_l^{(n-1)}\left(k(x_l,x)+\mu_p(x_l)\right)+\mu_p(x).
\end{align}

The whole algorithm is given in Alg.~\ref{alg:Lazy-MMD-FW}, where we consider the situation that we have already pre-processed particles via SVGD to further reduce the wall clock time. We can also consider the case that particles are not processed by SVGD. In that case, we simply initialize particles sampling from prior or randomly.

\begin{algorithm}[ht]\small
   \caption{Lazy LMO}
   \label{alg:lazy_oracle}
\begin{algorithmic}[1]
   \STATE {\bfseries Input:} $\Phi_n$, K, $\mu_{\hat{p}}^{(n)}$
   \STATE {\bfseries Output:} {\bfseries false} or $k(\cdot,y)$
   \IF{$x$ cached with $\mathrm{Dg}(\mu_{\hat{p}}^{(n)},x)\leq -\Phi/K$ exists}
   \STATE {\bfseries return} $k(x,\cdot) $\{Cache call\}
   \ELSE{}
   \STATE $k(\cdot,x)=$Cached approx-LMO($\mu_{\hat{p}}^{(n)}$)
   \IF{$\mathrm{Dg}(\mu_{\hat{p}}^{(n)},x)\leq -\Phi/K$}
   \STATE {\bfseries return} $k(x,\cdot)$ and {\bfseries add} $x$ to cache
   \ELSE{}
   \STATE {\bfseries return false}
   \ENDIF
   \ENDIF
\end{algorithmic}
\end{algorithm}

\begin{algorithm}[ht]\small
   \caption{Lazy MMD-FW}
   \label{alg:Lazy-MMD-FW}
\begin{algorithmic}[1]
   \STATE {\bfseries Input:} Accuracy parameter $K$, a target density $p(x)$,\\initial particles $\{x_n^{(0)}\}_{n=N}^n$ obtained by SVGD
   \STATE  Add all the initial particles into the cache.
   \STATE $x_{0}=\mathrm{argmin}_{x\in\mathrm{cache}}|\nabla_x\ln p(x)|$
   \STATE $\mu_{\hat{p}}^{(0)}=k(\cdot,x_{0})$
   \STATE $\Phi_{0}=-\mathrm{argmin}_{x\in\mathrm{cache}} \mathrm{Dg}(\mu_{\hat{p}}^{(0)},x)/2$
   \FOR{iteration $n$}
   \STATE $\bar{g}_n= $Lazy-LMO($\Phi_n,K,\mu_{\hat{p}}^{(n)}$)
   \IF{$\bar{g}_n=${\bfseries false}}
   \STATE $\mu_{\hat{p}}^{(n+1)}=\mu_{\hat{p}}^{(n)}$
   \STATE $\Phi_{n+1}=\frac{\Phi_n}{2}$
   \ELSE{}
   \STATE $\lambda _k=\mathrm{argmin}_{\lambda \in [0,1]} J((1-\lambda)\mu_{\hat{p}}^{(n)}+\lambda \bar{g}_n)$
   \STATE Update $\mu_{\hat{p}}^{(n+1)}=(1-\lambda _l)\mu_{\hat{p}}^{(n)}+\lambda _n \bar{g}_n$
   \STATE $\Phi_{n+1}=\Phi_n$
   \ENDIF
   \ENDFOR
\end{algorithmic}
\end{algorithm}

The theoretical behavior of this algorithm is as follows:
\begin{thm}$\mathrm{(Consistency)}$\label{thm:bound:2} Under the condition of Theorem 1 in the main paper, the error $|Z_{f,p}-Z_{f,\hat{p}}|$ of Alg.~3 in the main paper is bounded at the following rate:
\begin{align}
&|Z_{f,p}-Z_{f,\hat{p}}|\nonumber \\
&\leq\mathrm{MMD}(\{(w_n, x_n)\}_{n=1}^N) \nonumber \\
&\leq \begin{cases}
      \sqrt{2}r e^{-\frac{N}{2}\left(\frac{R\epsilon}{KrL}\right)^2},  & \text{$\mathcal{H}$ is the finite dimension} \\
      \frac{(\delta_w+\delta)r^2}{\delta(N\delta_w\delta+2)}, & \text{$\mathcal{H}$ is the infinite dimension}\\
    \end{cases}
\end{align}
where $r$ is the diameter of the marginal polytope $\mathcal{M}$, $\delta$ is the accuracy parameter, and $R$ is the radius of the smallest ball of center $\mu_p$ included $\mathcal{M}$ ($R$ is above 0 only when the dimension of $\mathcal{H}$ is finite.), and $\epsilon=\min_{k}\Phi_k$ which is positively bounded.
\end{thm}

\begin{thm}$\mathrm{(Contraction)}$ Let $S\subseteq \mathbb{R}$ be an open neighborhood of the true integral $Z_{f,p}$ and let $\gamma=\mathrm{inf}_{r\in S^c}|r-Z_{f,p}|>0$. Then the posterior probability of mass on $S^c=\mathbb{R}\setminus S$ by Alg~\ref{alg:Lazy-MMD-FW} vanishes at a rate:
\begin{align}
\mathrm{prob}(S^c)\leq\frac{2r}{\sqrt{\pi}\gamma}e^{-\frac{N}{2}\left(\frac{R\epsilon}{KdL}\right)^2-\frac{\gamma^2}{4r^2}e^{\left(\frac{R\epsilon}{KdL}\right)^2N}}\\
\text{for $\mathcal{H}$ is infinite dimension}\nonumber
\end{align}
where $r$ is the diameter of the marginal polytope $\mathcal{M}$, $\delta$ is the accuracy parameter, $R$ is the radius of the smallest ball of center $\mu_p$ included $\mathcal{M}$, and $\epsilon=\min_{k}\Phi_k$ which is positively bounded.
\end{thm}
Those proofs are shown later in this section.

Practically, we have to calculate Eq.~(\ref{duality_gap}) and this is difficult since this includes the integral $\mu_p$. We tried to approximate this term by the technique of biased importance sampling \cite{bamler2017perturbative}, but not work well. Thus, the practical implementation of this algorithm is the future work.

\subsection{Proofs}
The proof goes in the same way as MMD-FW and we have to be careful about how the approximate LMO describes. We use the proof of Proposition 3.2 in \cite{beck2004conditional} as we did in the proof of MMD-FW. The notation below is the same as \cite{beck2004conditional}.

In the Lazy-LMO, we do not solve the LMO for every iteration but first, we check whether the stored particles satisfies the following condition or not,
\begin{align}
\langle v_k,\tilde{w}_k\rangle\leq-\Phi_k/K.
\end{align}
If there exists the particle which satisfies above condition, then we do not solve the LMO but just return the particle which satisfies above condition. If no particle satisfies above condition, then we solve approximate LMO or update the accuracy parameter $\Phi_k$ following the algorithm. First, we consider the case that we did not update the accuracy parameter, that is, all the procedures consist of only positive calls\cite{lazy_FW}.

Here we assume that objective function is $L/M$ lipshitz, where we introduced $M$ to skip the rescaling the Lipschitz constant later. Then $\|v_k\|\leq L$ holds. Thus, $\frac{1}{L}\leq\frac{1}{\|v_k\|}$ holds.
When we consider the bound of $\|v_k\|^2$, then the discussion up to Eq.(12) in \cite{beck2004conditional} holds by replacing all the $w$ by $\tilde{w}$. And thus
\begin{align}
\|v_k\|^2=\frac{\|v_{k-1}\|^2\|\tilde{w}_{k-1}\|^2-\langle v_{k-1},\tilde{w}_{k-1}\rangle^2}{\|v_{k-1}-\tilde{w}_{k-1}\|^2}
\end{align}
holds. When we use $-\frac{1}{L^2}\geq-\frac{1}{\|v_k\|^2}$, following relation holds
\begin{align}
\|v_k\|^2&=\frac{\|v_{k-1}\|^2\|\tilde{w}_{k-1}\|^2-\langle v_{k-1},\tilde{w}_{k-1}\rangle^2}{\|v_{k-1}-\tilde{w}_{k-1}\|^2} \nonumber \\
&\leq \frac{\|v_{k-1}\|^2(\|\tilde{w}_{k-1}\|^2-\frac{\Phi^2_k}{(LK)^2}}{\|v_{k-1}-\tilde{w}_{k-1}\|^2} \nonumber \\
&\leq \left(1-\left(\frac{\Phi_k}{KL(\|g\|+\rho_s\|M\|)}\right)^2\right)\|v_{k-1}\|^2 \nonumber \\
&= \left(1-\left(\frac{\Phi_k}{C}\right)^2\right)\|v_{k-1}\|^2 
\end{align}
Here we assume that all the calls are only positive, that no negative call exists. Then following relation holds,
\begin{align}\label{App:B_lazy_bound}
\|v_k\|^2\leq\|v_{0}\|^2e^{-\frac{1}{C^2}\sum_{l=0}^k\Phi_k}\leq\|v_{0}\|^2e^{-\frac{k}{C^2}\min_k\Phi_k}
\end{align}
However as shown by Theorem4.1 in \cite{lazy_FW}, it is impossible to construct the algorithm that no negative call exists and the number of successive positive call are bounded.
Fortunately, now we want to bound the objective function by the number of particles not the iteration of the algorithm, this is different from \cite{lazy_FW}. And the number of particles increases only when the positive call is used. This means we can use the bound of Eq.(\ref{App:B_lazy_bound}) directly. And as shown in theorem4.1 in \cite{lazy_FW}, the number of the negative call is also bounded, that is the lazy algorithm surely increase the number of particles and the situation that no positive call exists does not occur. 

By definition, the value of $\min_k\Phi_k$ is positively bounded and thus the theorem directly obtained in the same way as the proof of MMD-FW.
Above proof is about the line search, but we can show the variant of constant step and fully correction easily by the same discussion with MMD-FW.

The value of Lipshitz constant is, our objective function is $J(\mu_{\hat{p}})=\frac{1}{2}\|\mu_p-\mu_{\hat{p}}\|_\mathcal{H}^2$, and thus
\begin{align}
L_k&=\max_{g \in \mathcal{M}}\langle \mu_{\hat{p}}^{(k)}-\mu_p,g \rangle \nonumber \\
&=\max_{x}\sum_{l=0}^{k}w_l^{(k)}k(x_l,x)-\int k(x',x)p(x)dx
\end{align}
and the lipschitz constant becomes $L=\max_k L_k$

About the contraction theorem, we can prove it in the same way as MMD-FW.

\section{Quadrature rules}\label{sup:preliminary}
\subsection{Herding and Quadrature}
When exact integration cannot be done, we often resort to use the quadrature rule approximations. A quadrature rules approximate the integral by weighted sum of functions at the certain points,
\begin{align}\label{integral_1}
\hat{Z}_{f,p}=\sum_{n=1}^N w_nf(x_n),
\end{align}
where we approximated $p(x)$ by $\hat{p}(x)=\sum_{n=1}^N w_n\delta(x_n)$ and $\delta(x_n)$ is a Dirac measure at $x_n$. There are many ways to specifying the combination of $\{(w_n,x_n)\}_{n=1}^N$. In this paper, we call $w_n$s as \emph{weights} and $x_n$s as \emph{particles}. Most widely used quadrature rule is the Monte Carlo(MC). We simply set all the $w_n=\frac{1}{N}$ and we produce $x_n$s by drawn from $p(x)$ randomly. This non-deterministic sampling based approximation converges at a rate $\mathcal{O}(\frac{1}{\sqrt{N}})$. On the other hand, in the Quasi Monte Carlo, we decide $x_n$s to directly minimize the some criterion.

In the kernel herding method \cite{chen2010super,bach_herding_equi}, the discrepancy measure is the Maximum Mean Discrepancy (MMD). Let $\mathcal{H}$ be a Hilbert space of functions equipped with the inner product $\langle\cdot, \cdot\rangle_\mathcal{H}$ and associated norm $\|\cdot\| _\mathcal{H}$. The MMD is defined by
\begin{align}\label{mmd}
\mathrm{MMD}(\{(w_n,x_n)\}_{n=1}^N)=\underset{f\in \mathcal{H}:\|f\|_\mathcal{H}=1}{\mathrm{sup}}|Z_{f,p}-\hat{Z}_{f,p}|
\end{align}
If we consider $\mathcal{H}$ be a reproducing kernel Hilbert space(RKHS) with a kernel $k$. In this setup, we can rewrite the MMD using $k(x,x')$ and set all the $w_i=\frac{1}{N}$,
\begin{align}\label{mmd_kernel}
&\mathrm{MMD}^2(\{(w_i=\frac{1}{N},x_i)\}_{i=1}^N)\nonumber \\
&=\underset{f\in \mathcal{H}:\|f\|_\mathcal{H}=1}{\mathrm{sup}}|Z_{f,p}-\hat{Z}_{f,p}|^2=\|\mu_p-\mu_{\hat{p}}\|_\mathcal{H}^2 \nonumber \\
&=\mathrm{Const.}-2\iint k(x,x')p(x)\hat{p}(x')dxdx'+ \nonumber\\
&\qquad\qquad\qquad\iint k(x,x')\hat{p}(x)\hat{p}(x')dxdx'\nonumber \\
&=\mathrm{Const.}-\frac{2}{N}\sum_{n=1}^N\int k(x,x_n)p(x)dx+\frac{1}{N^2}\sum_{n,m=1}^{N}k(x_n,x_m)
\end{align}where $\mu_p=\int k(\cdot,x)p(x)dx \in \mathcal{H}$. The herding algorithm greedily minimize the above discrepancy in the following way,
\begin{align}\label{herding_greedily}
x_{N+1} &\leftarrow \argmin_x 
[\mathrm{MMD}^2(\{(w_n=\frac{1}{N+1},x_n)\}_{n=1}^N, \nonumber\\
&\qquad\qquad\qquad (w_{N+1}=\frac{1}{N+1},x))]\nonumber \\
&=\argmax_x [\frac{2}{N+1}\int k(x,x')p(x')dx'- \nonumber\\
&\qquad\qquad\qquad\frac{2}{N+1}\sum_{n=1}^{N}k(x,x_n)]
\end{align}
It is widely known that, under certain assumption, they converges at a rate $\mathcal{O}(\frac{1}{N})$.

\end{document}